\newcommand*{\NIPS}{}

\newcommand*{\CAMREADY}{}
\newcommand*{\ARXIV}{}
\ifdefined\ICML
        \documentclass{article}
        \usepackage{times}
        \ifdefined\CAMREADY
                \usepackage[accepted]{icml2017} 
        \else
                \usepackage{icml2017}
        \fi
        \ifdefined\ARXIV
                \makeatletter
                \renewcommand{\ICML@appearing}{}
                \makeatother
        \fi
\fi

\ifdefined\NIPS
        \documentclass{article}
         \PassOptionsToPackage{numbers}{natbib}
        \ifdefined\CAMREADY
                \usepackage[final]{nips_2017} 
        \else
                \usepackage{nips_2017}
        \fi
        \ifdefined\ARXIV
                \makeatletter
                \renewcommand{\@noticestring}{}
                \makeatother
        \fi
\fi
\usepackage{color}
\usepackage{tabu,booktabs,multirow}
\usepackage{amsfonts}
\usepackage{amsmath}
\usepackage{algorithm}
\usepackage{algorithmic}
\usepackage{graphicx}
\usepackage{nicefrac}
\usepackage{bbm}
\usepackage{amsthm}
\usepackage{wrapfig}
\usepackage{diagbox}
\usepackage{hyperref}
\usepackage{subcaption}
\usepackage{caption}
\usepackage[titletoc,title]{appendix}
\usepackage{url}
\usepackage{wrapfig}

\newtheorem{definition}{Definition}
\newtheorem{lemma}{Lemma}
\newtheorem{corollary}{Corollary}
\newtheorem{theorem}{Theorem}

\newtheorem{claim}{Claim}

\newcommand{\m}{{\mathbf m}}
\newcommand{\x}{{\mathbf x}}

\newcommand{\vv}{{\mathbf v}}
\newcommand{\w}{{\mathbf w}}

\newcommand{\aaa}{{\mathbf a}}

\newcommand{\A}{{\mathcal A}}
\newcommand{\B}{{\mathcal B}}
\newcommand{\D}{{\mathcal D}}

\newcommand{\M}{{\mathcal M}}
\newcommand{\PP}{{\mathbb P}}

\newcommand{\X}{{\mathcal X}}
\newcommand{\Y}{{\mathcal Y}}

\newcommand{\R}{{\mathbb R}}
\newcommand{\N}{{\mathbb N}}

\newcommand{\Q}{{\mathcal Q}}

\newcommand{\mubf}{{\boldsymbol{\mu}}}
\newcommand{\sigmabf}{{\boldsymbol{\sigma}}}

\newcommand{\abs}[1]{\left\lvert#1 \right\rvert}

\newcommand{\rank}[1]{\textrm{rank}\left( #1 \right)}

\newcommand{\shorteq}{\scalebox{0.8}[1]{\text{$=$}}}
\DeclareMathOperator*{\argmax}{argmax} 
\DeclareMathOperator*{\argmin}{argmin}

\renewcommand{\bibsection}{}
\ifdefined\CAMREADY
        \newcommand{\githuburl}[1]{\url{https://github.com/HUJI-Deep/#1}}
\else
        \newcommand{\githuburl}[1]{\url{https://<anonymized>}}
\fi

\newcommand\gapaftersection{\vspace{0mm}}
\newcommand\gapaftersubsection{\vspace{0mm}}
\newcommand\gapbeforesection{\vspace{0mm}}
\newcommand\gapbeforesubsection{\vspace{0mm}}

\ifdefined\ICML
	\icmltitlerunning{Tensorial Mixture Models}
\fi

\begin{document}
% TITLE AND AUTHORS
\ifdefined\ICML
        \twocolumn[
        \icmltitle{Tensorial Mixture Models}
        
        
        \begin{icmlauthorlist}
        \icmlauthor{Or Sharir}{huji}
        \icmlauthor{Ronen Tamari}{huji}
        \icmlauthor{Nadav Cohen}{huji}
        \icmlauthor{Amnon Shashua}{huji}
        \end{icmlauthorlist}
        
        \icmlaffiliation{huji}{Hebrew University of Jerusalem, Israel}
        
        \icmlcorrespondingauthor{Or Sharir}{or.sharir@cs.huji.ac.il}
%        \icmlcorrespondingauthor{Ronen Tamari}{ronent@cs.huji.ac.il}
%        \icmlcorrespondingauthor{Nadav Cohen}{cohennadav@cs.huji.ac.il}
%        \icmlcorrespondingauthor{Amnon Shashua}{shashua@cs.huji.ac.il}
        
        \icmlkeywords{Deep Learning, Generative Models, Tractable Inference, Missing Data, Tensors}
        
        \vskip 0.3in
        ]

        \printAffiliationsAndNotice{}
\fi

\ifdefined\NIPS
        \title{Tensorial Mixture Models}
        \author{
          Or~Sharir\\
          Department of Computer Science\\
          The Hebrew University of Jerusalem\\
          Israel \\
          \texttt{or.sharir@cs.huji.ac.il} \\
          \And
          Ronen~Tamari\\
          Department of Computer Science\\
          The Hebrew University of Jerusalem\\
          Israel \\
          \texttt{ronent@cs.huji.ac.il} \\
          \And
          Nadav~Cohen\\
          Department of Computer Science\\
          The Hebrew University of Jerusalem\\
          Israel \\
          \texttt{cohennadav@cs.huji.ac.il} \\
          \And
          Amnon~Shashua\\
          Department of Computer Science\\
          The Hebrew University of Jerusalem\\
          Israel \\
          \texttt{shashua@cs.huji.ac.il} \\
        }
        \maketitle
\fi

% ABSTRACT
\begin{abstract}

Casting neural networks in generative frameworks is a highly sought-after endeavor these days.
Contemporary methods, such as Generative Adversarial Networks, capture some of the generative capabilities, but not all.
In particular, they lack the ability of tractable marginalization, and thus are not suitable for many tasks.
Other methods, based on arithmetic circuits and sum-product networks, do allow tractable marginalization, but their performance is challenged by the need to learn the structure of a circuit. Building on the tractability of arithmetic circuits, we leverage concepts from tensor analysis, and derive a family of generative models we call Tensorial Mixture Models (TMMs).
TMMs assume a simple convolutional network structure, and in addition, lend themselves to theoretical analyses that allow comprehensive understanding of the relation between their structure and their expressive properties.
We thus obtain a generative model that is tractable on one hand, and on the other hand, allows effective representation of rich distributions in an easily controlled manner.
These two capabilities are brought together in the task of classification under missing data, where TMMs deliver state of the art accuracies with seamless implementation and design.

\end{abstract}

% INTRODUCTION
\section{Introduction} \label{sec:intro}
\gapaftersection

There have been many attempts in recent years to marry generative models with neural networks,
including successful methods, such as Generative Adversarial Networks~\citep{Goodfellow:2014td},
Variational Auto-Encoders~\citep{Kingma:2013tz}, NADE~\citep{JMLR:v17:16-272}, and PixelRNN~\citep{vandenOord:2016um}.
Though each of the above methods has demonstrated its usefulness on some tasks, it is yet unclear if
their advantage strictly lies in their generative nature or some other attribute. More broadly, we
ask if combining generative models with neural networks could lead to methods who have a {clear advantage}
over purely discriminative models.

On the most fundamental level, if $X$ stands for an instance and~$Y$ for its class, generative models learn $\PP(X,Y)$, from which
we can also infer $\PP(Y|X)$, while discriminative models learn only $\PP(Y|X)$. It might not be immediately apparent if this sole
difference leads to any advantage. In~\citet{Ng:2001wg}, this question was studied with respect to the sample complexity, proving
that under \emph{some cases} it can be significantly lesser in favor of the generative classifier. We wish to highlight
a more clear cut case, by examining the problem of classification under missing data~--~where the value of some of the entries
of $X$ are unknown at prediction time. Under these settings, discriminative classifiers typically rely on some
form of data imputation, i.e. filling missing values by some auxiliary method prior to prediction. Generative classifiers,
on the other hand, are naturally suited to handle missing values through marginalization~--~effectively
assessing every possible completion of the missing values. Moreover, under mild assumptions, this method is optimal
\emph{regardless of the process by which values become missing} (see sec.~\ref{sec:missing_data}).

It is evident that such application of generative models assumes we can efficiently and exactly compute $\PP(X,Y)$, a process known as \emph{tractable inference}.
Moreover, it assumes we may efficiently marginalize over any subset of $X$, a procedure we refer to as \emph{tractable marginalization}. 
Not all generative models have both of these properties, and specifically
not the ones mentioned in the beginning of this section. Known models that do possess these properties, e.g. Latent Tree
Model~\citep{Mourad:2013kz}, have other limitations. A detailed discussion can be found in sec.~\ref{sec:related_works},
but in broad terms, all known generative models possess one of the following shortcomings: (i) they are insufficiently expressive to
model high-dimensional data (images, audio, etc.), (ii) they require explicitly designing all the dependencies of the data, or
(iii)~they do not have tractable marginalization. Models based on neural networks typically solve~(i) and~(ii), but are
incapable of~(iii), while more classical methods, e.g. mixture models, solve~(iii) but suffer from~(i) and~(ii). 

There is a long history of specifying tractable generative models through arithmetic circuits and sum-product 
networks~\citep{Darwiche:2003hx,Poon:2012vd}~--~computational graphs comprised solely of product and
weighted sum nodes. To address the shortcomings
above, we take a similar approach, but go one step further and leverage tensor analysis to distill it to a specific family
of models we call Tensorial Mixture Models. 
A Tensorial Mixture Model assumes a convolutional network structure, but as opposed to previous methods tying generative models with neural networks, lends itself to theoretical analyses that allow a thorough understanding of the relation between its structure and its expressive properties.
We thus obtain a generative model that is tractable on one hand, and on the other hand, allows effective representation of rich distributions in an easily controlled manner.

% PROBABILISTIC MODEL
\gapbeforesection
\section{Tensorial Mixture Models} \label{sec:model}
\gapaftersection

One of the simplest types of tractable generative models are mixture models, where the probability
distribution is defined שד the convex combination of $M$ mixing components % $\{\PP(\x|d;\theta_d)\}_{d=1}^M$
(e.g. Normal distributions): $\PP(\x) = \sum\nolimits_{d=1}^M \PP(d) \PP(\x|d;\theta_d)$.
Mixture models are very easy to learn, and many of them are able to approximate any probability
distribution, given sufficient number of components, rendering them suitable for a variety of tasks. 
The disadvantage of classic mixture models is that they do not scale will to high dimensional data (``curse of dimensionality'').
%Despite the advantages of mixture models, they do not scale well to high dimensional data. 
To address this challenge, we extend mixture models, leveraging the fact many high dimensional domains (e.g.~images) are typically comprised of small, simple local structures. 
We represent a high dimensional instance as $X = (\x_1,\ldots,\x_N)$~--~an $N$-length sequence
of $s$-dimensional vectors $\x_1,\ldots,\x_N~\in~\R^s$ (called \emph{local structures}). $X$ is typically thought of as
an image, where each local structure $\x_i$ corresponds to a local patch from that image, where no two patches are
overlapping. We assume that the distribution of individual local structures can be efficiently modeled by
some mixture model of few components, which for natural image patches, was shown to be the case~\citep{Zoran:2011jn}.
Formally, for all $i \in [N]$ there exists $d_i \in [M]$ such that $\x_i \sim P(\x|d_i;\theta_{d_i})$, where $d_i$
is a hidden variable specifying the matching component for the $i$-th local structure. The probability
density of sampling $X$ is thus described by:
\begin{align}\label{eq:tmm}
P(X) &= \sum\nolimits_{d_1,\ldots,d_N=1}^M P(d_1,\ldots,d_N) \prod\nolimits_{i=1}^N P(\x_i | d_i; \theta_{d_i})
\end{align} 
where $P(d_1,\ldots,d_N)$ represents the prior probability of assigning components
$d_1,\ldots,d_N$ to their respective local structures $\x_1,\ldots,\x_N$. As with classical mixture models,
any probability density function $\PP(X)$ could be approximated arbitrarily well by eq.~\ref{eq:tmm},
as $M \to \infty$ (see app.~\ref{app:universal}).

At first glance, eq.~\ref{eq:tmm} seems to be impractical, having an exponential number of terms. In the literature,
this equation is known as the ``Network Polynomial''~\citep{Darwiche:2003hx}, and the traditional method to
overcome its intractability is to express $P(d_1,\ldots,d_N)$ by an arithmetic circuit, or sum-product networks,
following certain constraints (decomposable and complete).
We augment this method by viewing $P(d_1,\ldots,d_N)$ from an algebraic perspective, treating it as a tensor of
order $N$ and dimension $M$ in each mode, i.e., as a multi-dimensional array, $\A_{d_1,\ldots,d_N}$ specified by $N$
indices $d_1,\ldots,d_N$, each ranging in $[M]$, where $[M]{\equiv}\{1,\ldots,M\}$.
We refer to $\A_{d_1,\ldots,d_N}{\equiv}P(d_1,\ldots,d_N)$ as the \emph{prior tensor}.
Under this perspective, eq.~\ref{eq:tmm} can be thought of as a mixture model with tensorial mixing weights, thus
we call the arising models \emph{Tensorial Mixture Models}, or TMMs for short.

\gapbeforesubsection
\subsection{Tensor Factorization, Tractability, and Convolutional Arithmetic Circuits}
\gapaftersubsection

Not only is it intractable to compute eq.~\ref{eq:tmm}, but it is also impossible to even store the prior tensor.
We argue that addressing the latter is intrinsically tied to addressing the former. For example, if we impose
a sparsity constraint on the prior tensor, then we only need to compute the few non-zero terms of eq.~\ref{eq:tmm}.
TMMs with sparsity constraints can represent common generative models, e.g. GMMs~(see app.~\ref{app:sparsity_example}).
However, they do not take full advantage of the prior tensor.
Instead, we consider constraining TMMs with prior tensors that adhere to non-negative low-rank factorizations.

We begin by examining the simplest case, where the prior tensor $\A$ takes a \emph{rank-1} form, i.e.
there exist vectors $\vv^{(1)},\ldots,\vv^{(N)} \in \R^M$ such that $\A_{d_1,\ldots,d_N} = \prod_{i=1}^N v^{(i)}_{d_i}$, or
in tensor product notation, $\A = \vv^{(1)} \otimes \cdots \otimes \vv^{(N)}$.
If we interpret\footnote{$\A$ represents a probability,
and w.l.o.g. we can assume all entries of $\vv^{(i)}$ are non-negative and $\sum_{d{=}1}^M v^{(i)}_d{\shorteq}1$}
$v^{(i)}_d = P(d_i{=}d)$ as a probability over $d_i$, and so $P(d_1,\ldots,d_N) = \prod_i P(d_i)$, then it reveals that imposing
a rank-1 constraint is actually equivalent to assuming the hidden variables $d_1,\ldots,d_N$ are statistically independent. Applying it to
eq.~\ref{eq:tmm} results in the tractable form $P(X)=\prod_{i=1}^N \sum_{d=1}^M P(d_i{=}d) P(\x_i|d_i,\theta_{d_i})$, or in
other words, a product of mixture models. Despite the familiar setting, this strict assumption severely limits expressivity.

\begin{figure}
\centering
\includegraphics[width=0.9\linewidth]{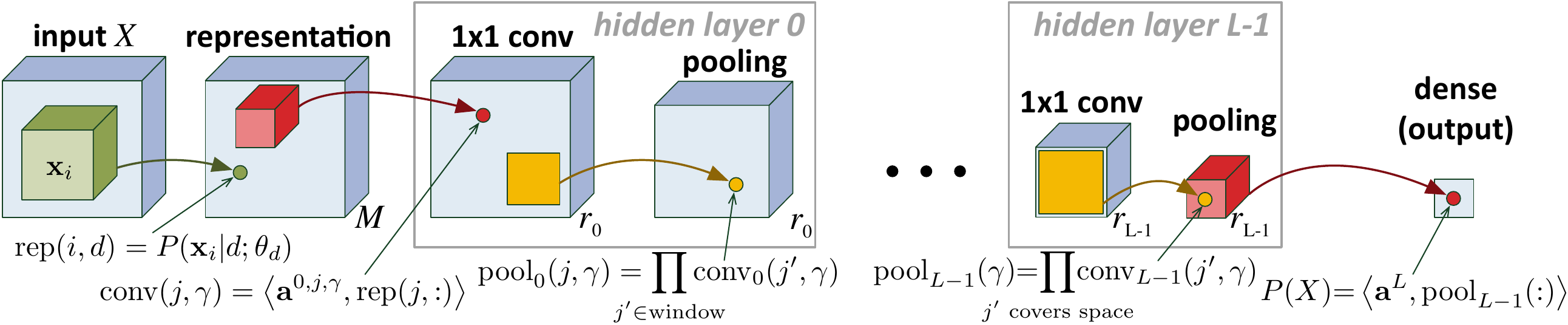}
\vspace{-1.5mm}
\caption{A generative variant of Convolutional Arithmetic Circuits.}
\label{fig:generative_convac}
\vspace{-1.5mm}
\end{figure}

In a broader setting, we look at general factorization schemes that given sufficient resources could represent any tensor.
Namely, the CANDECOMP/PARAFAC~(CP) and the Hierarchical Tucker~(HT) factorizations. The CP
factorization is simply a sum of rank-1 tensors, extending the previous case, and HT factorization can be seen as
a recursive application of CP (see def. in app.~\ref{app:tensor_background}). Since both factorization schemes are
solely based on product and weighted sum operations, they could be realized through arithmetic circuits. 
As shown by \citet{expressive_power}, this gives rise to a specific class of convolutional networks named
Convolutional Arithmetic Circuits~(ConvACs), which consist of $1{\times}1$-convolutions,
non-overlapping product pooling layers, and linear activations. More specifically, the CP factorization corresponds
to shallow ConvACs, HT corresponds to deep ConvACs, and the number of channels in each layer corresponds
to the respective concept of ``rank'' in each factorization scheme. 
In general, when a tensor factorization is applied to eq.~\ref{eq:tmm}, inference is equivalent to first computing the
likelihoods of all mixing components $\{P(\x_i|d;\theta_d)\}_{d{\shorteq}1,i{\shorteq}1}^{M,N}$, in what we call the \emph{representation} layer,
followed by a ConvAC. A complete network is illustrated in fig.~\ref{fig:generative_convac}.

\begin{wrapfigure}{r}{.3\columnwidth}
\centering
\includegraphics[width=\linewidth]{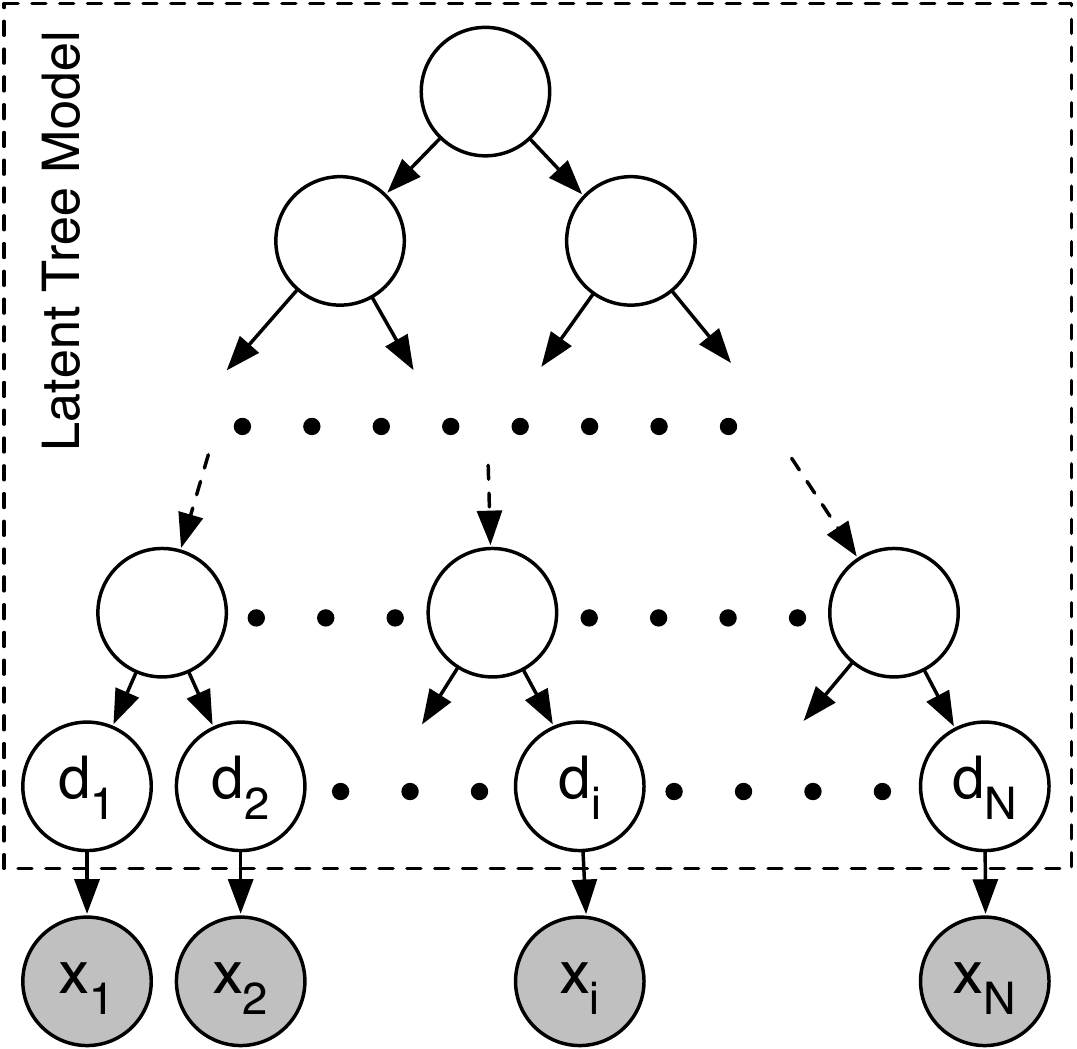}
\caption{Graphical model description of HT-TMM}
\label{fig:graphical_model}
\end{wrapfigure}

When restricting the prior tensor of eq.~\ref{eq:tmm} to a factorization, we must
ensure it represents actual probabilities, i.e. it is non-negative and its entries sum to one.
This can be addressed through a restriction to non-negative factorizations, which translates to
limiting the parameters of each convolutional kernel to the simplex.
There is a vast literature on the relations between non-negative factorizations and generative
models~\citep{Hofmann:1999vka,Mourad:2013kz}.
As opposed to most of these works, we apply factorizations merely to
derive our model and analyze its expressivity~--~not for learning its parameters~(see sec.~\ref{sec:training}).

From a generative perspective, the restriction of convolutional kernels to the simplex results in a latent tree graphical model, as illustrated in
fig.~\ref{fig:graphical_model}.  Each \emph{hidden layer} in the ConvAC network~--~a pair of convolution
and pooling operations, corresponds to a transition between two levels in the tree. More specifically, each level
is comprised of multiple latent variables, one for each spatial position in the input to a hidden layer in the network. Each
latent variable in the input to the $l$-th layer takes values in $[r_{l-1}]$~--~the number of channels in the layer that precedes it. 
Pooling operations in the network correspond to the parent-child relationships in the tree~--~a set of latent variables are siblings with
a shared parent in the tree, if they are positioned in the same pooling window in the network. The weights of convolution operations correspond to the
transition matrix between a parent and each of its children, i.e.~if $H_p$ is the parent latent variable, taking
values in $[r_l]$, and $H_{child}$ is one of its child variables, taking values in $[r_{l-1}]$, then
$P(H_{child}{\shorteq}i| H_p{\shorteq}c) \shorteq w^{(c)}_i$, where $\w^{(c)}$ is the $1{\times}1$ convolutional
kernel for the $c$-th output channel. With the above graphical representation in place, we can
easily draw samples from our model.

To conclude this subsection, by leveraging an algebraic perspective of the network polynomial (eq.~\ref{eq:tmm}), we show
that tractability is related to the tensor properties of the priors, and in particular, that low rank factorizations are equivalent
to inference via ConvACs. 
The application of arithmetic circuits to achieve tractability is by itself not a novelty.
However, the particular convolutional arithmetic circuits we propose lead to a comprehensive understanding of
representational abilities, and as a result, to a straightforward architectural design of TMMs.

\gapbeforesubsection
\subsection{Controlling the Expressivity and Inductive Bias of TMMs}\label{sec:theory}
\gapaftersubsection

As discussed in sec.~\ref{sec:intro}, it is not enough for a generative model to be tractable~--~it must
also be sufficiently expressive, and moreover, we must also be able to understand how its structure affects its
expressivity. In this section we explain how our algebraic perspective enables us to achieve this.

To begin with, since we derived our model by factorizing the prior tensor, it immediately follows that given sufficient
number of channels in the ConvAC, i.e.~given sufficient ranks in the tensor factorization, any distribution could be
approximated arbitrarily well (assuming~$M$ is allowed to grow).
In short, this amounts to saying that TMMs are universal.
Though many other generative models are known to be universal, it is typically not clear how
one may assess what a given structure of finite size can and cannot express. 
In contrast, the expressivity of ConvACs
has been throughly studied in a series of works~\citep{expressive_power,inductive_bias,Cohen:0ZJHmEow,Levine:2017wt},
each of which examined a different attribute of its structure. In \citet{expressive_power} it was proven that ConvACs
exhibit the Depth Efficiency property, i.e. deep networks are exponentially more expressive than shallow ones.
In \citet{inductive_bias} it was shown that deep networks can efficiently model some input correlations but not all,
and that by designing appropriate pooling schemes, different preferences may be encoded, i.e.~the inductive bias
may be controlled. In \citet{Cohen:0ZJHmEow} this result was extended to more complex connectivity patterns,
involving mixtures of pooling schemes. Finally, in \citet{Levine:2017wt}, an exact relation between the number of
channels and the correlations supported by a network has been found, enabling tight control over expressivity and
inductive bias. All of these results are brought forth by the relations of ConvACs to tensor factorizations.
They allow TMMs to be analyzed and designed in much more principled ways than alternative high-dimensional generative models.\footnote{
As a demonstration of the fact that ConvAC analyses are not affected by the non-negativity and normalization restrictions of our
generative variant, we prove in app.~\ref{app:depth_efficiency} that the Depth Efficiency property still holds.}

\gapbeforesubsection
\subsection{Classification and Learning}\label{sec:training}
\gapaftersubsection

TMMs realized through ConvACs, sharing many of the same traits as ConvNets, are especially
suitable to serve as classifiers. We begin by introducing a class variable $Y$, and
model the conditional likelihood $\PP(X|Y{\shorteq}y)$ for each $y\in [K]$. Though it is
possible to have separate generative models for each class, it is much more efficient to
leverage the relation to ConvNets and use a shared ConvAC instead, which is equivalent to a joint-factorization
of the prior tensors for all classes. This results in a single network, where instead of a single scalar
output representing $\PP(X)$, multiple outputs are driven by the network, representing $\PP(X|Y{\shorteq}y)$ for each class~$y$.
Predicting the class of a given instance is carried through Maximum A-Posteriori, i.e.~by returning the most likely class.
In the common setting of uniform class priors, i.e. $\PP(Y{\shorteq}y){\equiv}\frac{1}{K}$, this corresponds to classification by
maximal network output, as customary with ConvNets. We note that in practice, na\"{\i}ve implementation
of ConvACs is not numerically stable\footnote{Since high degree polynomials (as computed by
ACs) are susceptible to numerical underflow or overflow.}, and this is treated by performing all
computations in log-space, which transforms ConvACs into \emph{SimNets}~--~a recently introduced deep
learning architecture~\citep{simnets1,simnets2}. 

Suppose now that we are given a training set $S=\{(X^{(i)}{\in}(\R^s)^N,Y^{(i)}{\in}[K])\}_{i=1}^{|S|}$ of instances and labels,
and would like to fit the parameters $\Theta$ of our model according to the Maximum Likelihood principle, or equivalently,
by minimizing the Negative Log-Likelihood~(NLL) loss function: $\mathcal{L}(\Theta) = \mathbb{E}[-\log \PP_\Theta(X,Y)]$.
The latter can be factorized into two separate loss terms:
\begin{align*}
\mathcal{L}(\Theta) = \mathbb{E}[-\log \PP_\Theta(Y|X)] + \mathbb{E}[-\log \PP_\Theta(X)]
\end{align*}
where $\mathbb{E}[-\log \PP_\Theta(Y|X)]$, which we refer to as
the \emph{discriminative loss}, is commonly known as the cross-entropy loss,
and $\mathbb{E}[-\log \PP_\Theta(X)]$, which corresponds to maximizing the prior
likelihood $\PP(X)$, has no analogy in standard discriminative classification. 
It is this term that captures
the generative nature of the model, and we accordingly refer to it as the \emph{generative loss}.
Now, let $N_\Theta(X^{(i)};y){:=}\log \PP_\Theta(X^{(i)}|Y{=}y)$ stand for the $y$'th output of the SimNet
(ConvAC in log-space) realizing our model with parameters~$\Theta$.
In the case of uniform class
priors ($\PP(Y{=}y)\equiv\nicefrac{1}{K}$), the empirical estimation of $\mathcal{L}(\Theta)$ may be written as:
\begin{align}
\mathcal{L}(\Theta;S) = -\frac{1}{\abs{S}}\sum\nolimits_{i=1}^{\abs{S}} \log \frac{e^{N_\Theta(X^{(i)};Y^{(i)})}}{\sum\nolimits_{y=1}^{K}e^{N_\Theta(X^{(i)};y)}}
- \frac{1}{\abs{S}}\sum\nolimits_{i=1}^{\abs{S}} \log \sum\nolimits_{y=1}^{K}e^{N_\Theta(X^{(i)};y)}
\label{eq:objective}
\end{align}
This objective includes the standard softmax loss as its first term, and an additional generative loss as its second.
Rather than employing dedicated Maximum Likelihood methods for training (e.g. Expectation Minimization),
we leverage once more the resemblance between our networks and ConvNets, and optimize the above objective
using Stochastic Gradient Descent~(SGD).

% CLASSIFICATION UNDER MISSING DATA THROUGH MARGINALIZATION
\gapbeforesection
\section{Classification under Missing Data through Marginalization}\label{sec:missing_data}
\gapaftersection

A major advantage of generative models over discriminative ones lies in their ability to cope with missing data,
specifically in the context of classification. By and large, discriminative methods either attempt to complete missing
parts of the data before classification (a process known as \emph{data imputation}), or learn directly
to classify data with missing values~\citep{Little:2002uj}. The first of these approaches relies on the quality of
data completion, a much more difficult task than the original one of classification under missing data. Even if
the completion was optimal, the resulting classifier is known to be sub-optimal~(see app.~\ref{app:mbayes_proof}).
The second approach does not rely on data completion, but nonetheless assumes that the distribution of missing
values at train and test times are similar, a condition which often does not hold in practice. Indeed, \citet{Globerson:2006jv}
coined the term ``nightmare at test time'' to refer to the common situation where a classifier must cope with
missing data whose distribution is different from that encountered in training.

As opposed to discriminative methods, generative models are endowed with a natural mechanism for
classification under missing data. Namely, a generative model can simply marginalize over missing values,
effectively classifying under all possible completions, weighing each completion according to its probability. 
This, however, requires tractable inference and marginalization. We have already shown in sec.~\ref{sec:model}
that TMMs support the former, and will show in sec.~\ref{sec:missing_data:margin} that marginalization
can be just as efficient. Beforehand, we lay out the formulation of classification under missing data.

Let $\X$ be a random vector in~$\R^s$ representing an object, and let $\Y$ be a random variable in~$[K]$
representing its label. Denote by~$\D(\X,\Y)$ the joint distribution of~$(\X,\Y)$, and by
$(\x{\in}\R^s,y{\in}[K])$ specific realizations thereof. Assume that after sampling a specific instance $(\x,y)$, a
random binary vector $\M$ is drawn conditioned on $\X{=}\x$. More concretely, we sample a binary mask
$\m{\in}\{0,1\}^s$ (realization of~$\M$) according to a distribution $\Q(\cdot|\X{=}\x)$. $x_i$
is considered missing if~$m_i$ is equal to zero, and observed otherwise. Formally, we
consider the vector $\x{\odot}\m$, whose~$i$'th coordinate is defined to hold~$x_i$ if $m_i{=}1$, and the
wildcard~$*$ if $m_i{=}0$. The classification task is then to predict~$y$ given access solely to $\x{\odot}\m$.

Following the works of~\citet{Rubin:1976gv,Little:2002uj}, we consider three cases for the missingness
distribution $\Q(\M{=}\m|\X{=}\x)$: missing completely at random~(\emph{MCAR}), where~$\M$ is independent
of~$\X$, i.e.~$\Q(\M{=}\m|\X{=}\x)$ is a function of~$\m$ but not of~$\x$; missing at random~(\emph{MAR}),
where~$\M$ is independent of the missing values in~$\X$, i.e.~$\Q(\M{=}\m|\X{=}\x)$ is a function of both~$\m$
and $\x$, but is not affected by changes in~$x_i$ if~$m_i{=}0$; and missing not at random~(\emph{MNAR}),
covering the rest of the distributions for which~$\M$ depends on missing values in~$\X$, i.e.~$\Q(\M{=}\m|\X{=}\x)$
is a function of both~$\m$ and $\x$, which at least sometimes is sensitive to changes in~$x_i$ when~$m_i{=}0$.

Let $\PP$ be the joint distribution of the object~$\X$, label~$\Y$, and missingness mask~$\M$:
\begin{align*}
\PP(\X{\shorteq}\x,\Y{\shorteq}y,\M{\shorteq}\m) = \D\left(\X{\shorteq}\x, \Y{\shorteq}y\right) \cdot \Q(\M{\shorteq}\m|\X{\shorteq}\x)
\end{align*}
For given $\x{\in}\R^s$ and $\m{\in}\{0,1\}^s$, denote by $o(\x,\m)$ the event where the random vector~$\X$
coincides with~$\x$ on the coordinates~$i$ for which $m_i{=}1$. For example, if~$\m$ is an all-zero vector,
$o(\x,\m)$ covers the entire probability space, and if~$\m$ is an all-one vector, $o(\x,\m)$ corresponds to
the event $\X{=}\x$. With these notations in hand, we are now ready to characterize the optimal
predictor in the presence of missing data. The proofs are common knowledge, but provided in app.~\ref{app:mbayes_proof}
for completeness.
\begin{claim}\label{claim:optimal_rule}
For any data distribution~$\D$ and missingness distribution~$\Q$, the optimal classification rule in terms of 0-1 loss is given by
predicting the class $y\in[K]$, that maximizes $\PP(\Y{\shorteq}y|o(\x,\m))\cdot\PP(\M{\shorteq}\m|o(\x,\m),\Y{\shorteq}y)$, for 
an instance $\x {\odot} \m$.
\end{claim}
When the distribution~$\Q$ is MAR (or MCAR), the optimal classifier admits a simpler form, referred to as the \emph{marginalized Bayes predictor}:
\begin{corollary}\label{corollary:mar}
Under the conditions of claim~\ref{claim:optimal_rule}, if the distribution $\Q$ is MAR (or MCAR), the optimal classification rule may be written as:
\begin{equation}
h^*(\x \odot \m) = \argmax\nolimits_y~\PP(\Y{=}y|o(\x,\m))
\label{eq:mbayes}
\end{equation}
\end{corollary}

Corollary~\ref{corollary:mar} indicates that in the MAR setting, which is frequently encountered in practice,
optimal classification does \emph{not} require prior knowledge regarding the missingness distribution~$\Q$.
As long as one is able to realize the marginalized Bayes predictor (eq.~\ref{eq:mbayes}), or equivalently,
to compute the likelihoods of observed values conditioned on labels ($\PP(o(\x,\m)|Y{=}y)$), classification
under missing data is guaranteed to be optimal, regardless of the corruption process taking place. This is in
stark contrast to discriminative methods, which require access to the missingness distribution during training,
and thus are not able to cope with unknown conditions at test time.

Most of this section dealt with the task of prediction given an input with missing data, where we assumed
we had access to a ``clean'' training set, and only faced missingness during prediction.
However, many times we wish to tackle the reverse task, where the training set itself is riddled with
missing data. Tractability leads to an advantage here as well: under the MAR assumption,
learning from missing data with the marginalized likelihood objective results in an unbiased classifier~\citep{Little:2002uj}.

In the case of TMMs, marginalizing over missing values is just as efficient as plain inference~--~requires
only a single pass through the corresponding network. The exact mechanism is carried out in similar fashion
as in sum-product networks, and is covered in app.~\ref{sec:missing_data:margin}.  Accordingly, the marginalized
Bayes predictor (eq.~\ref{eq:mbayes}) is realized efficiently, and classification under missing data (in the MAR
setting) is optimal (under generative assumption), regardless of the missingness distribution.

\gapbeforesection
\section{Related Works} \label{sec:related_works}
\gapaftersection

There are many generative models realized through neural networks, and convolutional networks in particular,
e.g. Generative Adversarial Networks~\citep{Goodfellow:2014td}, Variational Auto-Encoders~\citep{Kingma:2013tz},
and NADE~\citep{JMLR:v17:16-272}. However, most do not posses tractable inference, and of the few that do,
non posses tractable marginalization over any set of variables. Due to limits of space, we defer the discussion
on the above to app.~\ref{app:extended_related_works}, and in the remainder of this section focus instead on
the most relevant works.

As mentioned in sec.~\ref{sec:model}, we build on the approach of specifying generative models through Arithmetic
Circuits~(ACs)~\citep{Darwiche:2003hx}, and specifically, our model is a strict subclass of the well-known Sum-Product
Networks~(SPNs)~\citep{Poon:2012vd}, under the decomposable and complete restrictions. Where our work
differs is in our algebraic approach to eq.~\ref{eq:tmm}, which gives rise to a specific structure of ACs, called ConvACs,
and a deep theory regarding their expressivity and inductive bias (see sec.~\ref{sec:theory}). 
In contrast to the structure we proposed, the current literature on general SPNs does
not prescribe any specific structures, and its theory is limited to either very specific instances~\citep{Delalleau:2011vh},
or very broad classes, e.g fixed-depth circuits~\citep{Martens:2014tr}. In the early works on SPNs, specialized networks
of complex structure were designed for each task based mainly on heuristics, often bearing little resemblance to common
neural networks. Contemporary works have since moved on to focus mainly on learning the structure
of SPNs directly from data~\citep{Peharz:2013cl,Gens:2013ufa,Adel:2015wf,Rooshenas:2014wb}, leading to improved results
in many domains. Despite that, only few published studies have applied this method to natural domains (images, audio, etc.),
on which only limited performance, compared to other common methods, was reported, specifically on the MNIST
dataset~\citep{Adel:2015wf}. The above suggests that choosing the right architecture of general SPNs, at least on some domains,
remains to be an unsolved problem. In addition, both the previously studied manually-designed SPNs, as well as ones with a learned
structure, lead to models, which according to recent works on GPU-optimized algorithms~\citep{maps-multi}, cannot be efficiently
implemented due to their irregular memory access patterns.  This is in stark contrast to our model, which leverages the same
patterns as modern ConvNets, and thus enjoys similar run-time performance. An additional difference in our
work is that we manage to successfully train our model using standard SGD. Even though this approach has already been
considered by~\citet{Poon:2012vd}, they deemed it lacking and advocated for specialized optimization algorithms instead.

Outside the realm of generative networks, tractable graphical models, e.g. Latent Tree Models~(LTMs)~\citep{Mourad:2013kz},
are the most common method for tractable inference. Similar to SPNs, it is not straightforward to find the proper structure
of graphical models for a particular problem, and most of the same arguments apply here as well. Nevertheless, it is
noteworthy that recent progress in structure and parameters learning of LTMs~\citep{Huang:2015tb,Anandkumar:2014uc}
was also brought forth by connections to tensor factorizations, similar to our approach. Unlike the aforementioned
algorithms, we utilize tensor factorizations solely for deriving our model and analyzing its expressivity, while leaving learning to
SGD~--~the most successful method for training neural networks. Leveraging their perspective to analyze
the optimization properties of our model is viewed as a promising avenue for future research.

\gapbeforesection
\section{Experiments} \label{sec:exp}
\gapaftersection

We demonstrate the properties of TMMs through both qualitative and
quantitative experiments. In sec.~\ref{subsec:exp:missing} we present state of the art results on image classification under missing data, with robustness
to various missingness distributions. In sec.~\ref{subsec:exp:timit} we show
that our results are not limited to images, by applying TMMs for speech recognition.
Finally, in app.~\ref{app:exp:vis} we show visualizations of samples drawn from TMMs, shedding light on their generative nature.
Our implementation, based on Caffe~\citep{Jia:2014up} and
MAPS~\citep{maps-multi} (toolbox for efficient GPU code generation), as well as all other code for reproducing our experiments, are available at: \githuburl{Generative-ConvACs}.
Extended details regarding the experiments are provided in app.~\ref{app:exp_details}. 

\gapbeforesubsection
\subsection{Image Classification under Missing Data}\label{subsec:exp:missing}
\gapaftersubsection

\begin{table}
\centering
\resizebox{0.5\columnwidth}{!}{%
\begin{tabular}{cccccccc}
\toprule
n=  & 0 & 25 & 50 & 75 & 100 & 125 & 150 \\ 
\midrule
LP    & 97.9 & 97.5 & 96.4 & 94.1 & 89.2 & 80.9 & 70.2 \\
HT-TMM & \textbf{98.5} & \textbf{98.2} & \textbf{97.8} & \textbf{96.5} & \textbf{93.9} & \textbf{87.1} & \textbf{76.3} \\
\bottomrule
\end{tabular}
}
\caption{Prediction for each two-class task of MNIST digits, under feature deletion noise.}
\label{table:exp_shamir}
\end{table}

\begin{figure}
\centering
\begin{subtable}{0.4\columnwidth}
\resizebox{\columnwidth}{!}{%
\begin{tabular}{lccccccc}
\toprule
\diagbox[width=5em, height= 1.7em]{$p_{\textrm{train}}$}{$p_{\textrm{test}}$} & 0.25 & 0.50 & 0.75 & 0.90 & 0.95 & 0.99 \\ 
\midrule
0.25           & 98.9             & 97.8             & 78.9             & 32.4             & 17.6             & 11.0 \\
0.50           & \textbf{99.1} & 98.6             & 94.6             & 68.1             & 37.9             & 12.9 \\
0.75           & 98.9             & \textbf{98.7} & \textbf{97.2} & 83.9             & 56.4             & 16.7 \\
0.90           & 97.6             & 97.5             & 96.7             & \textbf{89.0} & 71.0             & 21.3 \\
0.95           & 95.7             & 95.6             & 94.8             & 88.3             & \textbf{74.0} & 30.5 \\
0.99           & 87.3             & 86.7             & 85.0             & 78.2             & 66.2             & \textbf{31.3} \\
\midrule
i.i.d. (rand) & 98.7             & 98.4             & 97.0             & 87.6             & 70.6             & 29.6 \\
rects (rand)& 98.2             & 95.7             & 83.2             & 54.7             & 35.8             & 17.5 \\
\bottomrule
\end{tabular}}
\caption{MNIST with i.i.d. corruption}
\label{table:convnet_iid_bias}
\end{subtable}
\quad\quad
\begin{subfigure}{0.4\columnwidth}
\includegraphics[width=\columnwidth]{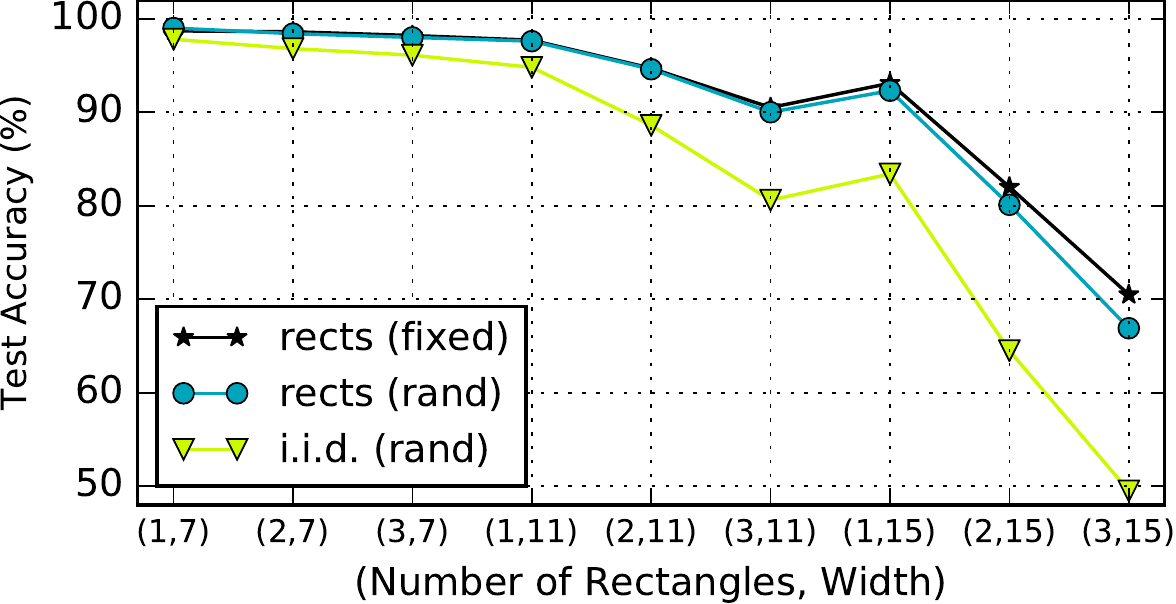}
\caption{MNIST with missing rectangles.}
\label{fig:convnet_mnist_rects}
\end{subfigure}
\caption{We examine ConvNets trained on one missingness distribution while tested on others.
             ``(rand)'' denotes training on distributions with randomized parameters.
              \textbf{(\subref{table:convnet_iid_bias})}~i.i.d. corruption: trained with probability $p_{\textrm{train}}$
              and tested on $p_{\textrm{test}}$.
              \textbf{(\subref{fig:convnet_mnist_rects})}~missing rectangles: training on randomized distributions (rand) compared
              to training on the same (fixed) missing rectangles distribution.
              }
\label{fig:convnet}
\end{figure}

\begin{figure}
\centering
\begin{subfigure}{0.48\textwidth}
\includegraphics[width=\textwidth]{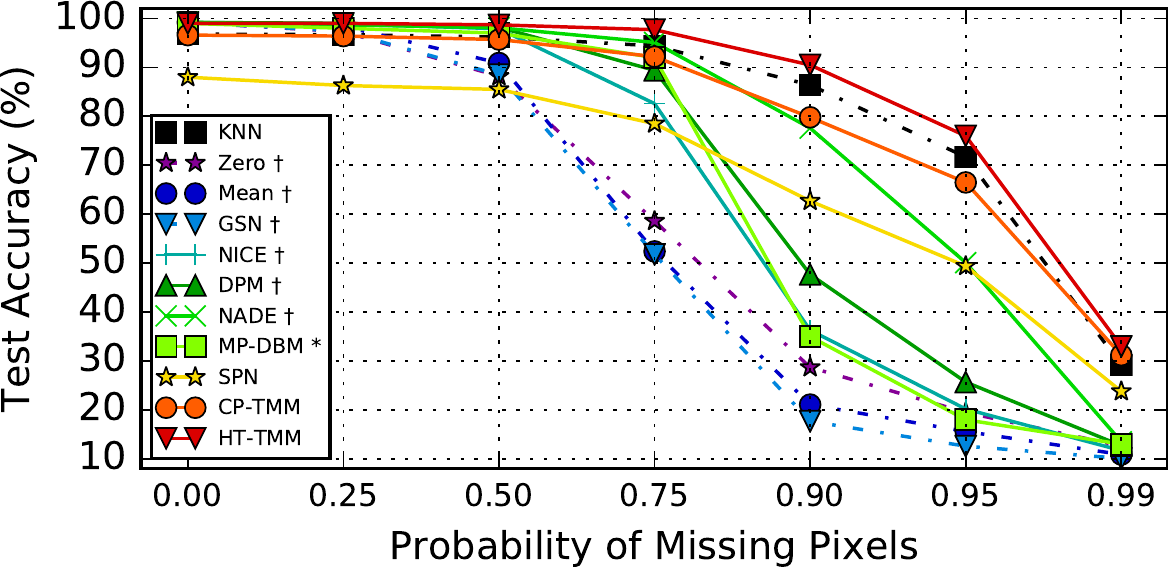}
\caption{MNIST with i.i.d. corruption.}
\label{fig:mnist_iid}
\end{subfigure}
~\quad
\begin{subfigure}{0.48\textwidth}
\includegraphics[width=\textwidth]{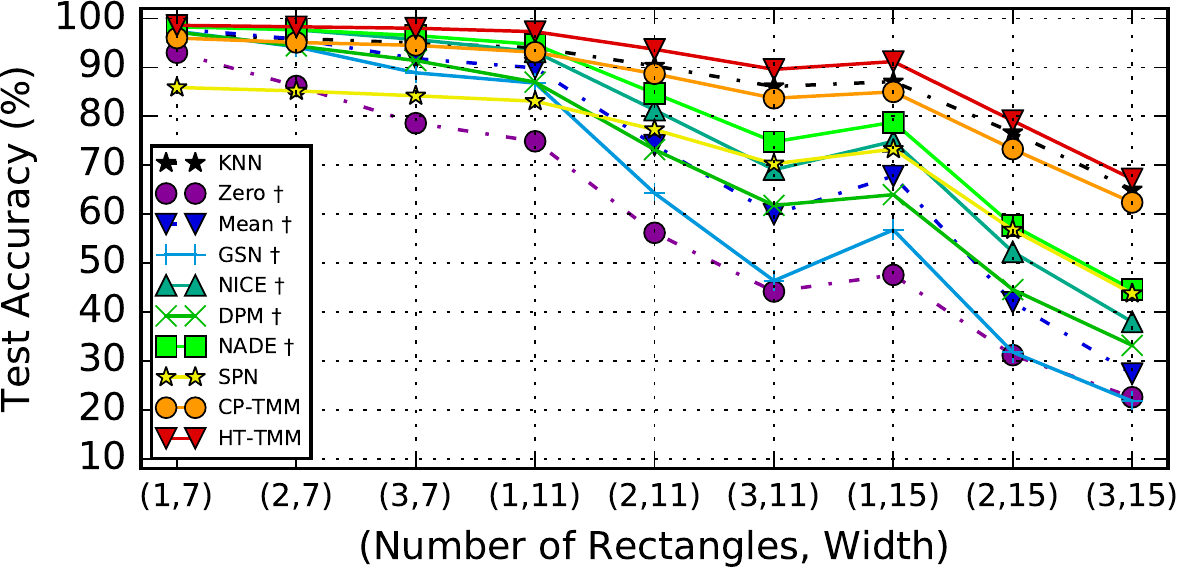}
\caption{MNIST with missing rectangles.}
\label{fig:mnist_rects}
\end{subfigure}

\begin{subfigure}{0.48\textwidth}
\includegraphics[width=\textwidth]{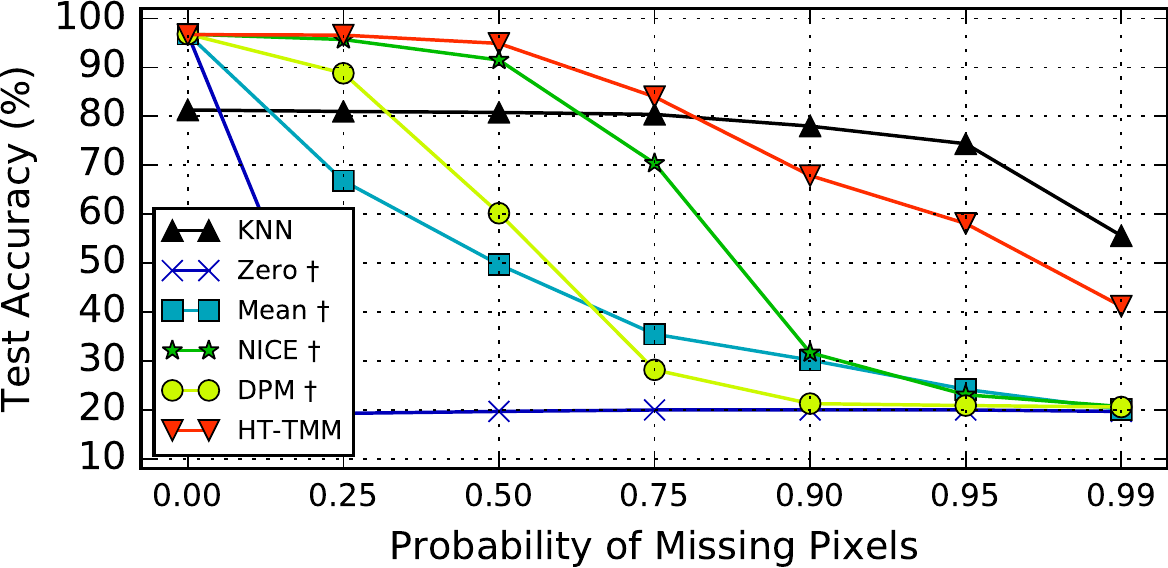}
\caption{NORB with i.i.d. corruption.}
\label{fig:norb_iid}
\end{subfigure}
~\quad
\begin{subfigure}{0.48\textwidth}
\includegraphics[width=\textwidth]{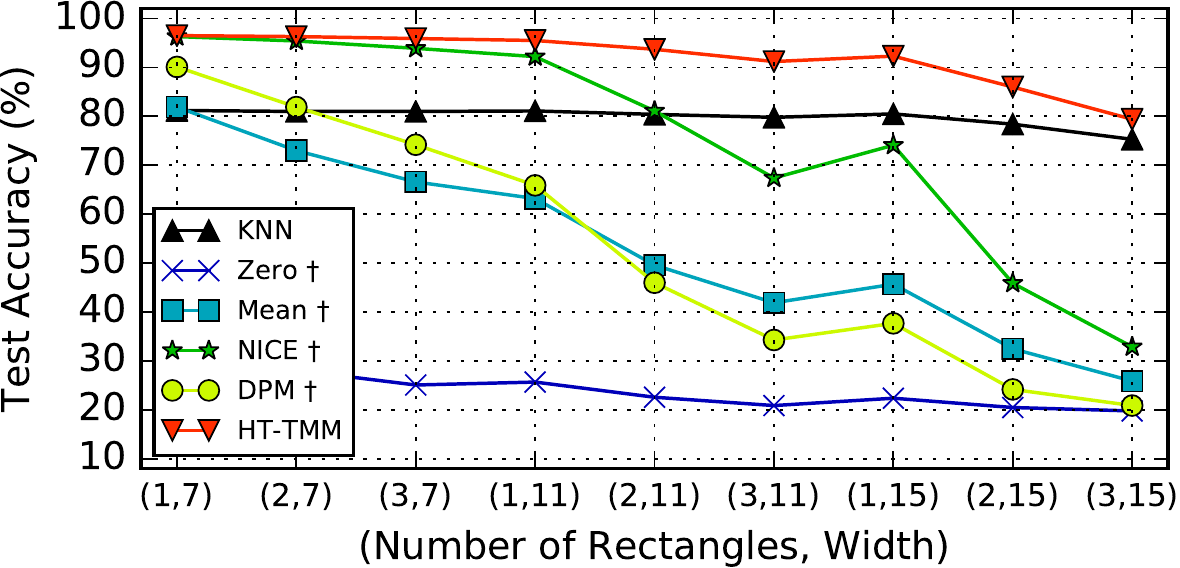}
\caption{NORB with missing rectangles.}
\label{fig:norb_rects}
\end{subfigure}
\caption{Blind classification under missing data. \textbf{(\subref{fig:mnist_iid},\subref{fig:norb_iid})}~Testing
               i.i.d. corruption with probability $p$ for each pixel.
              \textbf{(\subref{fig:mnist_rects},\subref{fig:norb_rects})}~Testing missing rectangles
              corruption with $n$ missing rectangles, each of width and hight equal to $W$.
              (*)~Based on the published results~\citep{Goodfellow:2013vm}.
             (\dag)~Data imputation algorithms.}
\label{fig:exp:multiclass}
\end{figure}

In this section we experiment on two datasets: MNIST~\citep{LeCun:1998hy} for digit classification,
and small NORB~\citep{LeCun:2004wl} for 3D object recognition. In our results, we refer to models
using shallow networks as CP-TMM, and to those using deep networks as HT-TMM, in accordance with the respective tensor
factorizations~(see sec.~\ref{sec:model}). The theory discussed in sec.~\ref{sec:theory} guided
our exact choice of architectures.  Namely, we used the fact~\citep{Levine:2017wt} that the capacity to model
either short- or long-range correlations in the input, is related to the number of channels in the beginning or end of
a network, respectively. In MNIST, discriminating between digits has more to do with long-range correlations
than the basic strokes digits are made of, hence we chose to start with few channels and end with many~--~layer widths were set to 64-128-256-512.
In contrast, the classes of NORB differ in much finer details, requiring
more channels in the first layers, hence layer widths were set to 256-256-256-512. 
In both cases, $M=32$ Gaussian mixing components were used.

We begin by comparing our generative approach to missing data against classical methods,
namely, methods based on \citet{Globerson:2006jv}. They regard missing data as ``feature
deletion'' noise, replace missing entries by zeros, and devise a learning algorithm
over linear predictors that takes the number of missing features, $n$, into account.
The method was later improved by \citet{Dekel:2008jb}.
We compare TMMs to the latter, with $n$ non-zero pixels randomly chosen and changed to zero,
in the two-class prediction task derived from each pair of MNIST digits.
Due to limits of their implementation, only 300 images per digit are used for training.
Despite this, and the fact that the evaluated scenario is of the MNAR type (on which optimality is not guaranteed~--~see sec.~\ref{sec:missing_data}), we achieve
significantly better results (see table~\ref{table:exp_shamir}), and unlike their method,
which requires several classifiers and knowing $n$, we use a single TMM
with no prior knowledge.

Heading on to multi-class prediction under missing data, we focus on the challenging ``blind'' setting, where the missingness distribution at test time
is completely unknown during training. We simulate two kinds of MAR missingness distributions: (i)~an
i.i.d. mask with a fixed probability~$p\in[0,1]$ of dropping each pixel, and (ii)~a mask composed of the
union of $n$ (possibly overlapping) rectangles of width and height~$W$, each positioned randomly in the image (uniform distribution). 
We first demonstrate that purely discriminative
classifiers cannot generalize to all missingness distributions, by training the standard LeNeT
ConvNet~\citep{LeCun:1998hy} on one set of distributions and then testing it on others (see
fig.~\ref{fig:convnet}). Next, we present our main results. We compare our model against three different
approaches. First, as a baseline, we use K-Nearest Neighbors~(KNN) to vote on the most likely class,
augmented with an $l^2$-metric that disregards missing coordinates. KNN actually scores better than
most methods, but its missingness-aware distance metric prevents the common memory and runtime
optimizations, making it impractical for real-world settings. Second, we test various data-imputation methods, ranging from simply filling missing pixels with zeros or their mean, to modern generative models suited
to inpainting.  Data imputation is followed by a ConvNet prediction on the completed image. In general, we find that this approach only
works well when few pixels are missing. Finally, we test generative classifiers other than our model,
including MP-DBM and SPN (sum-product networks). MP-DBM is notable for being limited to approximations, and its results show
the importance of using exact inference instead. For SPN, we have augmented the model from~\citet{Poon:2012vd}
with a class variable $Y$, and trained it to maximize the joint probability $P(X,Y)$ using the code of~\citet{Zhao:2016va}.
The inferior performance of SPN suggests that the structure of TMMs, which are in fact a special case, is advantageous. Due to limitations
of available public code and time, not all methods were tested on all datasets and distributions.
See fig.~\ref{fig:exp:multiclass} for the complete results.

To conclude, TMMs significantly outperform all other methods tested on image classification with missing data.
Although they are a special case of SPNs, their particular structure appears to be more effective than ones existing in the literature.
We attribute this superiority to the fact that their architectural design is backed by comprehensive theoretical studies (see sec.~\ref{sec:theory}).

\gapbeforesubsection
\subsection{Speech Recognition under Missing Data}\label{subsec:exp:timit}
\gapaftersubsection

To demonstrate the versatility of TMMs, we also conducted limited experiments
on the TIMIT speech recognition dataset, following the same protocols as in sec.~\ref{subsec:exp:missing}.
We trained a TMM and a standard ConvNet on 256ms windows of raw data at 16Hz sample rate
to predict the phoneme at the center of a window. Both the TMM and the ConvNet
reached $78\%$ accuracy on the clean dataset, but when half of the audio is missing i.i.d.,
accuracy of the ConvNet with mean imputation drops to $34\%$, while the TMM remains at $63\%$. Utilizing common audio inpainting methods~\citep{Adler:2012dd} only improves accuracy of the ConvNet
to $48\%$, well below that of TMM.

\gapbeforesection
\section{Summary} \label{sec:summary}
\gapaftersection

This paper focuses on generative models which admit tractable inference and marginalization, capabilities that lie outside the realm of contemporary neural network-based generative methods.
We build on
prior works on tractable models based on arithmetic circuits and sum-product networks,
and leverage concepts from tensor analysis to derive a sub-class of models we call Tensorial Mixture
Models (TMMs). In contrast to existing methods, our algebraic approach
leads to a comprehensive understanding of the relation between model structure and representational properties.
In practice, utilizing this understanding for the design of TMMs has led to state of the art performance
in classification under missing data. We are currently investigating several avenues for future research,
including semi-supervised learning, and examining more intricate ConvAC architectures, such as the ones
suggested by \citet{Cohen:0ZJHmEow}).

\newcommand{\acknowledgments}{This work is supported by Intel grant ICRI-CI \#9-2012-6133, by ISF Center grant 1790/12 and by the European Research Council (TheoryDL project). Nadav Cohen is supported by a Google Fellowship in Machine Learning.}
\ifdefined\CAMREADY
	\subsubsection*{Acknowledgments}
	\acknowledgments
\fi

% REFERENCES
\subsubsection*{References}
\small{
\bibliographystyle{plainnat}
\bibliography{refs.bib}

\begin{thebibliography}{42}
\providecommand{\natexlab}[1]{#1}
\providecommand{\url}[1]{\texttt{#1}}
\expandafter\ifx\csname urlstyle\endcsname\relax
  \providecommand{\doi}[1]{doi: #1}\else
  \providecommand{\doi}{doi: \begingroup \urlstyle{rm}\Url}\fi

\bibitem[Adel et~al.(2015)Adel, Balduzzi, and Ghodsi]{Adel:2015wf}
Tameem Adel, David Balduzzi, and Ali Ghodsi.
\newblock {Learning the Structure of Sum-Product Networks via an SVD-based
  Algorithm.}
\newblock \emph{UAI}, 2015.

\bibitem[Adler et~al.(2012)Adler, Emiya, Jafari, and Elad]{Adler:2012dd}
A~Adler, V~Emiya, M~G Jafari, and M~Elad.
\newblock {Audio inpainting}.
\newblock \emph{IEEE Trans. on Audio, Speech and Language Processing},
  20:\penalty0 922--932, March 2012.

\bibitem[Anandkumar et~al.(2014)Anandkumar, Ge, Hsu, Kakade, and
  Telgarsky]{Anandkumar:2014uc}
Animashree Anandkumar, Rong Ge, Daniel Hsu, Sham~M Kakade, and Matus Telgarsky.
\newblock {Tensor decompositions for learning latent variable models.}
\newblock \emph{Journal of Machine Learning Research ()}, 15\penalty0
  (1):\penalty0 2773--2832, 2014.

\bibitem[Ben-Nun et~al.(2015)Ben-Nun, Levy, Barak, and Rubin]{maps-multi}
Tal Ben-Nun, Ely Levy, Amnon Barak, and Eri Rubin.
\newblock {Memory Access Patterns: The Missing Piece of the Multi-GPU Puzzle}.
\newblock In \emph{Proceedings of the International Conference for High
  Performance Computing, Networking, Storage and Analysis}, pages 19:1--19:12.
  ACM, 2015.

\bibitem[Bengio et~al.(2014)Bengio, Thibodeau-Laufer, Alain, and
  Yosinski]{Bengio:2013vx}
Yoshua Bengio, {\'E}ric Thibodeau-Laufer, Guillaume Alain, and Jason Yosinski.
\newblock {Deep Generative Stochastic Networks Trainable by Backprop}.
\newblock In \emph{International Conference on Machine Learning}, 2014.

\bibitem[Caron and Traynor(2005)]{caron2005zero}
Richard Caron and Tim Traynor.
\newblock {The Zero Set of a Polynomial}.
\newblock \emph{WSMR Report 05-02}, 2005.

\bibitem[Cohen and Shashua(2014)]{simnets1}
Nadav Cohen and Amnon Shashua.
\newblock {SimNets: A Generalization of Convolutional Networks}.
\newblock In \emph{Advances in Neural Information Processing Systems NIPS, Deep
  Learning Workshop}, 2014.

\bibitem[Cohen and Shashua(2017)]{inductive_bias}
Nadav Cohen and Amnon Shashua.
\newblock {Inductive Bias of Deep Convolutional Networks through Pooling
  Geometry}.
\newblock In \emph{International Conference on Learning Representations ICLR},
  April 2017.

\bibitem[Cohen et~al.(2016{\natexlab{a}})Cohen, Sharir, and
  Shashua]{expressive_power}
Nadav Cohen, Or~Sharir, and Amnon Shashua.
\newblock {On the Expressive Power of Deep Learning: A Tensor Analysis}.
\newblock In \emph{Conference on Learning Theory COLT}, May 2016{\natexlab{a}}.

\bibitem[Cohen et~al.(2016{\natexlab{b}})Cohen, Sharir, and Shashua]{simnets2}
Nadav Cohen, Or~Sharir, and Amnon Shashua.
\newblock {Deep SimNets}.
\newblock In \emph{Computer Vision and Pattern Recognition CVPR}, May
  2016{\natexlab{b}}.

\bibitem[Cohen et~al.(2017)Cohen, Tamari, and Shashua]{Cohen:0ZJHmEow}
Nadav Cohen, Ronen Tamari, and Amnon Shashua.
\newblock {Boosting Dilated Convolutional Networks with Mixed Tensor
  Decompositions}.
\newblock \emph{arXiv.org}, 2017.

\bibitem[Darwiche(2003)]{Darwiche:2003hx}
Adnan Darwiche.
\newblock {A differential approach to inference in Bayesian networks}.
\newblock \emph{Journal of the ACM (JACM)}, 50\penalty0 (3):\penalty0 280--305,
  May 2003.

\bibitem[Dekel and Shamir(2008)]{Dekel:2008jb}
Ofer Dekel and Ohad Shamir.
\newblock {Learning to classify with missing and corrupted features}.
\newblock In \emph{International Conference on Machine Learning}. ACM, 2008.

\bibitem[Delalleau and Bengio(2011)]{Delalleau:2011vh}
Olivier Delalleau and Yoshua Bengio.
\newblock {Shallow vs. Deep Sum-Product Networks.}
\newblock \emph{Advances in Neural Information Processing Systems}, pages
  666--674, 2011.

\bibitem[Dinh et~al.(2014)Dinh, Krueger, and Bengio]{Dinh:2014vu}
Laurent Dinh, David Krueger, and Yoshua Bengio.
\newblock {NICE: Non-linear Independent Components Estimation}.
\newblock \emph{arXiv.org}, October 2014.

\bibitem[Gens and Domingos(2013)]{Gens:2013ufa}
R~Gens and P~M Domingos.
\newblock {Learning the Structure of Sum-Product Networks.}
\newblock \emph{Internation Conference on Machine Learning}, 2013.

\bibitem[Globerson and Roweis(2006)]{Globerson:2006jv}
Amir Globerson and Sam Roweis.
\newblock {Nightmare at test time: robust learning by feature deletion}.
\newblock In \emph{International Conference on Machine Learning}. ACM, 2006.

\bibitem[Goodfellow et~al.(2013)Goodfellow, Mirza, Courville, and
  Bengio]{Goodfellow:2013vm}
Ian Goodfellow, Mehdi Mirza, Aaron Courville, and Yoshua Bengio.
\newblock {Multi-Prediction Deep Boltzmann Machines}.
\newblock \emph{Advances in Neural Information Processing Systems}, 2013.

\bibitem[Goodfellow et~al.(2014)Goodfellow, Pouget-Abadie, Mirza, Xu,
  Warde-Farley, Ozair, Courville, and Bengio]{Goodfellow:2014td}
Ian Goodfellow, Jean Pouget-Abadie, Mehdi Mirza, Bing Xu, David Warde-Farley,
  Sherjil Ozair, Aaron Courville, and Yoshua Bengio.
\newblock {Generative Adversarial Nets}.
\newblock \emph{Advances in Neural Information Processing Systems}, 2014.

\bibitem[Hackbusch and K{\"u}hn(2009)]{Hackbusch:2009jj}
W~Hackbusch and S~K{\"u}hn.
\newblock {A New Scheme for the Tensor Representation}.
\newblock \emph{Journal of Fourier Analysis and Applications}, 15\penalty0
  (5):\penalty0 706--722, 2009.

\bibitem[Hofmann(1999)]{Hofmann:1999vka}
Thomas Hofmann.
\newblock \emph{{Probabilistic latent semantic analysis}}.
\newblock Morgan Kaufmann Publishers Inc., July 1999.

\bibitem[Huang et~al.(2015)Huang, N, Perros, Chen, Sun, and
  Anandkumar]{Huang:2015tb}
Furong Huang, Niranjan~U N, Ioakeim Perros, Robert Chen, Jimeng Sun, and Anima
  Anandkumar.
\newblock {Scalable Latent Tree Model and its Application to Health Analytics}.
\newblock In \emph{NIPS Machine Learning for Healthcare Workshop}, 2015.

\bibitem[Jia et~al.(2014)Jia, Shelhamer, Donahue, Karayev, Long, Girshick,
  Guadarrama, and Darrell]{Jia:2014up}
Yangqing Jia, Evan Shelhamer, Jeff Donahue, Sergey Karayev, Jonathan Long,
  Ross~B Girshick, Sergio Guadarrama, and Trevor Darrell.
\newblock {Caffe: Convolutional Architecture for Fast Feature Embedding.}
\newblock \emph{CoRR abs/1202.2745}, cs.CV, 2014.

\bibitem[Kingma and Welling(2014)]{Kingma:2013tz}
Diederik~P Kingma and Max Welling.
\newblock {Auto-Encoding Variational Bayes}.
\newblock In \emph{International Conference on Learning Representations}, 2014.

\bibitem[LeCun et~al.(1998)LeCun, Bottou, Bengio, and Haffner]{LeCun:1998hy}
Yan LeCun, Leon Bottou, Yoshua Bengio, and Patrick Haffner.
\newblock {Gradient-based learning applied to document recognition}.
\newblock \emph{Proceedings of the IEEE}, 86\penalty0 (11):\penalty0
  2278--2324, 1998.

\bibitem[LeCun et~al.(2004)LeCun, Huang, and Bottou]{LeCun:2004wl}
Yann LeCun, Fu~Jie Huang, and L{\'e}on Bottou.
\newblock {Learning Methods for Generic Object Recognition with Invariance to
  Pose and Lighting.}
\newblock \emph{Computer Vision and Pattern Recognition}, 2004.

\bibitem[Levine et~al.(2017)Levine, Yakira, Cohen, and Shashua]{Levine:2017wt}
Yoav Levine, David Yakira, Nadav Cohen, and Amnon Shashua.
\newblock {Deep Learning and Quantum Entanglement: Fundamental Connections with
  Implications to Network Design}.
\newblock \emph{arXiv.org}, April 2017.

\bibitem[Little and Rubin(2002)]{Little:2002uj}
Roderick J~A Little and Donald~B Rubin.
\newblock \emph{{Statistical analysis with missing data (2nd edition)}}.
\newblock John Wiley {\&} Sons, Inc., September 2002.

\bibitem[Martens and Medabalimi(2014)]{Martens:2014tr}
James Martens and Venkatesh Medabalimi.
\newblock {On the Expressive Efficiency of Sum Product Networks.}
\newblock \emph{CoRR abs/1202.2745}, cs.LG, 2014.

\bibitem[Mourad et~al.(2013)Mourad, Sinoquet, Zhang, Liu, and
  Leray]{Mourad:2013kz}
Rapha{\"e}l Mourad, Christine Sinoquet, Nevin~Lianwen Zhang, Tengfei Liu, and
  Philippe Leray.
\newblock {A Survey on Latent Tree Models and Applications.}
\newblock \emph{J. Artif. Intell. Res. ()}, cs.LG:\penalty0 157--203, 2013.

\bibitem[Ng and Jordan(2002)]{Ng:2001wg}
Andrew~Y Ng and Michael~I Jordan.
\newblock {On Discriminative vs. Generative Classifiers: A comparison of
  logistic regression and naive Bayes}.
\newblock In \emph{Advances in Neural Information Processing Systems NIPS, Deep
  Learning Workshop}, 2002.

\bibitem[Pedregosa et~al.(2011)Pedregosa, Varoquaux, Gramfort, Michel, Thirion,
  Grisel, Blondel, Prettenhofer, Weiss, Dubourg, Vanderplas, Passos,
  Cournapeau, Brucher, Perrot, and Duchesnay]{scikit-learn}
F~Pedregosa, G~Varoquaux, A~Gramfort, V~Michel, B~Thirion, O~Grisel, M~Blondel,
  P~Prettenhofer, R~Weiss, V~Dubourg, J~Vanderplas, A~Passos, D~Cournapeau,
  M~Brucher, M~Perrot, and E~Duchesnay.
\newblock {Scikit-learn: Machine Learning in Python}.
\newblock \emph{Journal of Machine Learning Research ()}, 12:\penalty0
  2825--2830, 2011.

\bibitem[Peharz et~al.(2013)Peharz, Geiger, and Pernkopf]{Peharz:2013cl}
Robert Peharz, Bernhard~C Geiger, and Franz Pernkopf.
\newblock {Greedy Part-Wise Learning of Sum-Product Networks}.
\newblock In \emph{Machine Learning and Knowledge Discovery in Databases},
  pages 612--627. Springer Berlin Heidelberg, Berlin, Heidelberg, September
  2013.

\bibitem[Poon and Domingos(2011)]{Poon:2012vd}
Hoifung Poon and Pedro Domingos.
\newblock {Sum-Product Networks: A New Deep Architecture}.
\newblock In \emph{Uncertainty in Artificail Intelligence}, 2011.

\bibitem[Rooshenas and Lowd(2014)]{Rooshenas:2014wb}
Amirmohammad Rooshenas and Daniel Lowd.
\newblock {Learning Sum-Product Networks with Direct and Indirect Variable
  Interactions.}
\newblock \emph{ICML}, 2014.

\bibitem[Rubin(1976)]{Rubin:1976gv}
Donald~B Rubin.
\newblock {Inference and missing data}.
\newblock \emph{Biometrika}, 63\penalty0 (3):\penalty0 581--592, December 1976.

\bibitem[Sohl-Dickstein et~al.(2015)Sohl-Dickstein, Weiss, Maheswaranathan, and
  Ganguli]{SohlDickstein:2015vq}
Jascha Sohl-Dickstein, Eric~A Weiss, Niru Maheswaranathan, and Surya Ganguli.
\newblock {Deep Unsupervised Learning using Nonequilibrium Thermodynamics.}
\newblock \emph{Internation Conference on Machine Learning}, 2015.

\bibitem[Uria et~al.(2016)Uria, {\'e}, Gregor, Murray, and
  Larochelle]{JMLR:v17:16-272}
Benigno Uria, Marc-Alexandre C {\^o}~t {\'e}, Karol Gregor, Iain Murray, and
  Hugo Larochelle.
\newblock {Neural Autoregressive Distribution Estimation}.
\newblock \emph{Journal of Machine Learning Research ()}, 17\penalty0
  (205):\penalty0 1--37, 2016.

\bibitem[van~den Oord et~al.(2016)van~den Oord, Kalchbrenner, and
  Kavukcuoglu]{vandenOord:2016um}
Aaron van~den Oord, Nal Kalchbrenner, and Koray Kavukcuoglu.
\newblock {Pixel Recurrent Neural Networks}.
\newblock In \emph{International Conference on Machine Learning}, 2016.

\bibitem[Zeiler and Fergus(2014)]{Zeiler:2014fra}
Matthew~D Zeiler and Rob Fergus.
\newblock {Visualizing and Understanding Convolutional Networks}.
\newblock In \emph{European Conference on Computer Vision}. Springer
  International Publishing, 2014.

\bibitem[Zhao et~al.(2016)Zhao, Poupart, and Gordon]{Zhao:2016va}
Han Zhao, Pascal Poupart, and Geoff Gordon.
\newblock {A Unified Approach for Learning the Parameters of Sum-Product
  Networks}.
\newblock In \emph{Advances in Neural Information Processing Systems NIPS, Deep
  Learning Workshop}, 2016.

\bibitem[Zoran and Weiss(2011)]{Zoran:2011jn}
Daniel Zoran and Yair Weiss.
\newblock {From learning models of natural image patches to whole image
  restoration.}
\newblock \emph{ICCV}, pages 479--486, 2011.

\end{thebibliography}
}

% APPENDIXES
\clearpage
\appendix

\section{The Universality of Tensorial Mixture Models}\label{app:universal}
In this section we prove the universality property of Generative ConvACs, as discussed in sec.~\ref{sec:model}.
We begin by taking note from functional analysis and define a new property called \emph{PDF total set},
which is similar in concept to a \emph{total set}, followed by proving that this property is invariant under the
cartesian product of functions, which entails the universality of these models as a corollary.

\begin{definition}
Let $\mathcal{F}$ be a set of PDFs over $\R^s$. $\mathcal{F}$ is
PDF total iff for any PDF $h(\x)$ over $\R^s$ and for all $\epsilon > 0$ there exists
$M~\in~\N$, $\{f_1(\x),\ldots,f_M(\x)\} \subset \mathcal{F}$ and
$\w \in \triangle^{M-1}$ s.t. $\left\| h(\x) - \sum_{i=1}^M w_i f_i(\x) \right\|_1 < \epsilon$.
In other words, a set is a PDF total set if its convex span is a dense set under $L^1$ norm.
\end{definition}

\begin{claim}
Let $\mathcal{F}$ be a set of PDFs over $\R^s$ and let
$\mathcal{F}^{\otimes N} = \{\prod_{i=1}^N f_i(\x) | \forall i, f_i(\x) \in \mathcal{F} \}$ be a set of PDFs
over the product space $(\R^s)^N$. If $\mathcal{F}$ is a PDF total set then $\mathcal{F}^{\otimes N}$
is PDF total set.
\end{claim}

\begin{proof} If $\mathcal{F}$ is the set of Gaussian PDFs over $\R^s$ with diagonal covariance matrices,
which is known to be a PDF total set, then $\mathcal{F}^{\otimes N}$ is the set of Gaussian PDFs over
$(\R^s)^N$ with diagonal covariance matrices and the claim is trivially true.

Otherwise, let $h(\x_1,\ldots,\x_N)$ be a PDF over $(\R^s)^N$ and let $\epsilon~>~0$. From the above,
there exists $K~\in~\N$, $\w~\in~\triangle^{M_1 - 1}$ and a set of diagonal Gaussians $\{g_{ij}(\x)\}_{i \in [M_1], j \in [N]}$ s.t.
\begin{align}\label{eq:univ:gaussian}
\left\| g(\x) - \sum_{i=1}^{M_1} w_i \prod_{j=1}^N g_{ij}(\x_j) \right\|_1 < \frac{\epsilon}{2}
\end{align}
Additionally, since $\mathcal{F}$ is a PDF total set then there exists
$M_2 \in \N$, $\{ f_k(\x)\}_{k\in[M_2]} \subset \mathcal{F}$ and
$\{\w_{ij} \in \triangle^{M_2 -1}\}_{i \in [M_1], j \in [N]}$ s.t. for all $i \in [M_1], j \in [N]$ it holds that
$\left\| g_{ij}(\x) - \sum_{k=1}^{M_2} w_{ijk} f_k(\x)\right\|_1 < \frac{\epsilon}{2N}$, from which it is trivially proven
using a telescopic sum and the triangle inequality that:
\begin{align}\label{eq:univ:gaussian_approx}
\left\|  \sum_{i=1}^{M_1} w_i \prod_{j=1}^N g_{ij}(\x) -  \sum_{i=1}^{M_1} w_i \prod_{j=1}^N \sum_{k=1}^{M_2} w_{ijk} f_k(\x_j) \right\|_1 &< \frac{\epsilon}{2}
\end{align}
From eq.~\ref{eq:univ:gaussian}, eq.~\ref{eq:univ:gaussian_approx} the triangle inequality it holds that:
\begin{align*}
\left\| g(\x) - \sum_{k_1,\ldots,k_N=1}^{M_2} \A_{k_1,\ldots,k_N} \prod_{j=1}^N f_{k_j}(\x_j) \right\|_1 &< \epsilon
\end{align*}
where $\A_{k_1,\ldots,k_N} = \sum_{i=1}^{M_1} w_i \prod_{j=1}^N w_{ijk_j}$ which holds $\sum_{k_1,\ldots,k_N=1}^{M_2} \A_{k_1,\ldots,k_N} = 1$.
Taking $M~=~M_2^N$, $\{\prod_{j=1}^N f_{k_j}(\x_j)\}_{k_1 \in [M_2],\ldots, k_N \in [M_2]} \subset \mathcal{F}^{\otimes N}$ and $\w = \textrm{vec}(\A)$
completes the proof.
\end{proof}

\begin{corollary}
Let $\mathcal{F}$ be a PDF total set of PDFs over $\R^s$, then the family of Generative ConvACs with
mixture components from $\mathcal{F}$ can approximate any $PDF$ over $(\R^s)^N$ 
arbitrarily well, given arbitrarily many components.
\end{corollary}

\section{TMMs with Sparsity Constraints Can Represent Gaussian Mixture Models} \label{app:sparsity_example}

As discussed in sec.~\ref{sec:model}, TMMs become tractable when a sparsity constraint is imposed on
the priors tensor, i.e. most of the entries of the tensors are replaced with zeros. In this section, we demonstrate
that under such a case, TMMs can represent Gaussian Mixture Models with diagonal covariance matrices, probably the most common type of mixture models.

With the same notations as sec.~\ref{sec:model}, assume the number of mixing components of the TMM is $M = N \cdot K$
for some $K \in \N$, let $\{\mathcal{N}(\x; \mubf_{ki}, \textrm{diag}(\sigmabf^2_{ki}))\}_{k,i}^{K,N}$ be these components, and finally, assume the prior tensor has the following structure:
\begin{align*}
P(d_1,\ldots,d_N) &=
       \begin{cases}
              w_k & \forall i\in [N], \,\,\,d_i {=} N {\cdot} (k{-}1) {+} i \\
              0 & \textrm{Otherwise}
       \end{cases}
%P(\x|d;\theta_d) &= \mathcal{N}(\x; \mubf_{ki}, \textrm{diag}(\sigmabf^2_{ki})), \,\,\, d {=} N {\cdot} (k{-}1) {+} i
\end{align*}
then eq.~\ref{eq:tmm} reduces to:
\begin{align*}
P(X) &= \sum\nolimits_{k=1}^K w_k \prod\nolimits_{i=1}^N  \mathcal{N}(\x; \mubf_{ki}, \textrm{diag}(\sigmabf^2_{ki}))
        = \sum\nolimits_{k=1}^K w_k \mathcal{N}(\x; \tilde{\mubf}_k, \textrm{diag}(\tilde{\sigmabf}^2_k)) \\
  \tilde{\mubf}_k &= (\mubf_{k1}^T, \ldots, \mubf_{kN}^T)^T \quad\quad\quad
  \tilde{\sigmabf}^2_k = ((\sigmabf^2_{k1})^T, \ldots, (\sigmabf^2_{kN})^T)^T
\end{align*}
which is equivalent to a diagonal GMM with mixing weights $\w \in \triangle^{K-1}$ (where $\triangle^{K-1}$ is the
$K$-dimensional simplex) and Gaussian mixture components with means $\{\tilde{\mubf}_k\}_{k=1}^K$
and covariances $\{\textrm{diag}(\tilde{\sigmabf}^2_k)\}_{k=1}^K$.

\section{Background on Tensor Factorizations} \label{app:tensor_background}

In this section we establish the minimal background in the field of tensor analysis required for following our work.
A tensor is best thought of as a multi-dimensional array $\A_{d_1,\ldots,d_N}\in\R$, where $\forall i\in[N], d_i \in [M_i]$.
The number of indexing entries in the array, which are also called \emph{modes}, is referred to as the \emph{order} of
the tensor. The number of values an index of a particular mode can take is referred to as the \emph{dimension} of the
mode. The tensor $\A \in \R^{M_1 \otimes \ldots \otimes M_N}$ mentioned above is thus of order $N$ with dimension
$M_i$ in its $i$-th mode. For our purposes we typically assume that $M_1 = \ldots = M_N = M$, and simply denote
it as $\A \in (\R^M)^{\otimes N}$.

The fundamental operator in tensor analysis is the \emph{tensor product}.
The tensor product operator, denoted by $\otimes$, is a generalization of outer product of vectors (1-ordered vectors) to
any pair of tensors. Specifically, let $\A$ and $\B$ be tensors of order $P$ and $Q$ respectively, then the tensor product
$\A \otimes \B$ results in a tensor of order $P+Q$, defined by:
$(\A \otimes \B)_{d_1,\ldots,d_{P+Q}} = \A_{d_1,\ldots,d_P} \cdot \B_{d_{P+1},\ldots,d_{P+Q}}$.

The main concept from tensor analysis we use in our work is that of tensor decompositions. The most straightforward
and common tensor decomposition format is the rank-1 decomposition, also known as a CANDECOMP/PARAFAC
decomposition, or in short, a \emph{CP decomposition}. The CP decomposition is a natural extension of low-rank
matrix decomposition to general tensors, both built upon the concept of a linear combination of rank-1 elements.
Similarly to matrices, tensors of the form $\vv^{(1)} \otimes \cdots \otimes \vv^{(N)}$, where $\vv^{(i)} \in \R^{M_i}$ are
non-zero vectors, are regarded as $N$-ordered rank-1 tensors, thus the rank-$Z$ CP decomposition of a tensor $\A$
is naturally defined by:
\begin{align} \label{eq:cp_decomp}
\A &= \sum_{z=1}^Z a_z \aaa^{z,1} \otimes \cdots \otimes \aaa^{z,N} \nonumber \\
\Rightarrow \A_{d_1,\ldots,d_N} &= \sum_{z=1}^Z a_z \prod_{i=1}^N a_{d_i}^{z,i}
\end{align}
where $\{\aaa^{z,i} \in \R^{M_i}\}_{i=1,z=1}^{N,Z}$ and $\aaa \in \R^Z$ are the parameters of the decomposition.
As mentioned above, for $N=2$ it is equivalent to low-order matrix factorization.
It is simple to show that any tensor $\A$ can be represented by the CP decomposition for some $Z$, where the minimal
such $Z$ is known as its \emph{tensor rank}.

Another decomposition we will use in this paper is of a hierarchical nature and known as the Hierarchical Tucker decomposition \citep{Hackbusch:2009jj},
which we will refer to as \emph{HT decomposition}.
While the CP decomposition
combines vectors into higher order tensors in a single step, the HT decomposition does that more gradually, combining
vectors into matrices, these matrices into 4th ordered tensors and so on recursively in a hierarchically fashion. Specifically,
the following describes the recursive formula of the HT decomposition\footnote{
More precisely, we use a special case of the canonical HT decomposition as presented in \citet{Hackbusch:2009jj}. In the
terminology of the latter, the matrices $A^{l,j,\gamma}$ are diagonal and equal to $diag(\aaa^{l,j,\gamma})$ (using the
notations from eq.~\ref{eq:ht_decomp}).} for a tensor
$\A \in (\R^M)^{\otimes N}$ where $N = 2^L$, i.e. $N$ is a power of two\footnote{The requirement for $N$ to be a power
of two is solely for simplifying the definition of the HT decomposition. More generally, instead of defining it through a complete
binary tree describing the order of operations, the canonical decomposition can use any balanced binary tree.}:
\begin{align}
\phi^{1,j,\gamma} &= \sum_{\alpha=1}^{r_0} a_\alpha^{1,j,\gamma} 
\aaa^{0,2j-1,\alpha} \otimes  \aaa^{0,2j,\alpha} 
\nonumber \\
&\cdots
\nonumber\\
\phi^{l,j,\gamma} &= \sum_{\alpha=1}^{r_{l-1}} a_\alpha^{l,j,\gamma} 
\underbrace{\phi^{l-1,2j-1,\alpha}}_{\text{order $2^{l-1}$}} \otimes  
\underbrace{\phi^{l-1,2j,\alpha}}_{\text{order $2^{l-1}$}} 
\nonumber\\
&\cdots 
\nonumber\\
\phi^{L-1,j,\gamma} &= \sum_{\alpha=1}^{r_{L-2}} a_\alpha^{L-1,j,\gamma} 
\underbrace{\phi^{L-2,2j-1,\alpha}}_{\text{order $\frac{N}{4}$}} \otimes  
\underbrace{\phi^{L-2,2j,\alpha}}_{\text{order $\frac{N}{4}$}}  
\nonumber\\ 
\A &= \sum_{\alpha=1}^{r_{L-1}} a_\alpha^L 
\underbrace{\phi^{L-1,1,\alpha}}_{\text{order $\frac{N}{2}$}} \otimes  
\underbrace{\phi^{L-1,2,\alpha}}_{\text{order $\frac{N}{2}$}}  
\label{eq:ht_decomp}
\end{align}
where the parameters of the decomposition are the vectors
$\{\aaa^{l,j,\gamma}{\in}\R^{r_{l-1}}\}_{l\in\{0,\ldots,L-1\}, j\in[\nicefrac{N}{2^l}], \gamma \in [r_l]}$
and the top level vector $\aaa^L \in \R^{r_{L-1}}$, and the scalars $r_0,\ldots,r_{L-1} \in \N$ are referred to as the
\emph{ranks of the decomposition}. Similar to the CP decomposition, any tensor can be represented by
an HT decomposition. Moreover, any given CP decomposition can be converted to an HT decomposition
by only a polynomial increase in the number of parameters. 

Finally, since we are dealing with generative models, the tensors we study are non-negative and
sum to one, i.e. the vectorization of $\A$ (rearranging its entries to the shape of a vector), denoted
by $\textrm{vec}(\A)$, is constrained to lie in the multi-dimensional simplex, denoted by:
\begin{align}\label{eq:simplex}
\triangle^k &:= \left\{\x \in \R^{k+1} | \sum\nolimits_{i=1}^{k+1} x_i = 1, \forall i \in [k+1]: x_i \geq 0\right\}
\end{align}

\section{Proof for the Depth Efficiency of Convolutional Arithmetic Circuits with Simplex Constraints}\label{app:depth_efficiency}

In this section we prove that the depth efficiency property of ConvACs that was proved in~\citet{expressive_power}
applies also to the generative variant of ConvACs we have introduced in sec.~\ref{sec:model}. Our analysis relies
on basic knowledge of tensor analysis and its relation to ConvACs, specifically, that the concept of ``ranks'' of each
factorization scheme is equivalent to the number of channels in these networks. For completeness, we provide a short
introduction to tensor analysis in app.~\ref{app:tensor_background}. The 

We prove the following theorem, which is the generative analog of theorem~1 from~\citep{expressive_power}:
\begin{theorem} \label{thm:tensor_rank}
Let $\A^y$ be a tensor of order $N$ and dimension $M$ in each mode, generated by the recursive formulas in
eq.~\ref{eq:ht_decomp}, under the simplex constraints introduced in sec.~\ref{sec:model}. Define
$r~:=~\min\{r_0,M\}$, and consider the space of all possible configurations for the parameters of the
decomposition~--~$\{\aaa^{l,j,\gamma}~\in~\triangle^{r_{l-1}-1}\}_{l,j,\gamma}$. In this space, the generated tensor $\A^y$
will have CP-rank of at least $r^{\nicefrac{N}{2}}$ almost everywhere (w.r.t. the product measure of simplex spaces).
Put differently, the configurations for which the CP-rank of $\A^y$ is less than $r^{\nicefrac{N}{2}}$ form a set of
measure zero. The exact same result holds if we constrain the composition to be ``shared'', i.e. set
$\aaa^{l,j,\gamma}\equiv\aaa^{l,\gamma}$ and consider the space of
$\{\aaa^{l,\gamma}~\in~\triangle^{r_{l-1}-1}\}_{l,\gamma}$ configurations.
\end{theorem}

The only differences between ConvACs and their generative counter-parts are the simplex constraints applied
to the parameters of the models, which necessitate a careful treatment to the measure theoretical arguments of
the original proof. More specifically, while the $k$-dimensional simplex $\triangle^{k}$ is a subset of the
$k+1$-dimensional space~$\R^{k+1}$, it has a zero measure with respect to the Lebesgue measure over
$\R^{k+1}$. The standard method to define a measure over $\triangle^k$ is by the Lebesgue measure over $\R^k$
of its projection to that space, i.e. let $\lambda:\R^k \to \R$ be the Lebesgue measure over
$\R^k$, $p:\R^{k+1} \to \R^k, p(\x) = (x_1, \ldots, x_k)^T$ be a projection, and $A \subset \triangle^k$ be a subset 
of the simplex, then the latter's measure is defined as $\lambda(p(A))$. Notice that $p(\triangle^k)$ has a positive
measure, and moreover that $p$ is invertible over the set $p(\triangle^k)$, and that its inverse is given by
$p^{-1}(x_1,\ldots,x_k) = (x_1,\ldots,x_k,1-\sum_{i=1}^k x_i)$. In our case, the parameter space is the cartesian
product of several simplex spaces of different dimensions, for each of them the measure is defined as above, and
the measure over their cartesian product is uniquely defined by the product measure. Though standard, the choice of
the projection function $p$ above could be seen as a limitation, however, the set of zero measure sets in $\triangle^k$
is identical for any reasonable choice of a projection $\pi$ (e.g. all polynomial mappings). More specifically, for any projection
$\pi:\R^{k+1}\to\R^k$ that is invertible over $\pi(\triangle^k)$, $\pi^{-1}$ is differentiable, and the Jacobian of $\pi^{-1}$
is bounded over $\pi(\triangle^k)$, then a subset $A \subset \triangle^k$ is of measure zero w.r.t. the projection $\pi$
iff it is of measure zero w.r.t. $p$ (as defined above).  
This implies that if we sample the weights of the generative decomposition (eq.~\ref{eq:ht_decomp} with simplex constraints)
by a continuous distribution, a property that holds with probability 1 under the standard parameterization (projection $p$), will
hold with probability 1 under any reasonable parameterization.

We now state and prove a lemma that will be needed for our proof of theorem~\ref{thm:tensor_rank}.
\begin{lemma}\label{lemma:rank_everywhere}
Let $M, N, K \in \N$, $1 \leq r \leq \min\{M,N\}$ and a polynomial mapping $A:\R^K \to \R^{M \times N}$
(i.e. for every $i \in [M],j\in [N]$ then $A_{ij}:\R^k \to \R$ is a polynomial function).
If there exists a point $\x \in \R^K$ s.t. $\rank{A(\x)} \geq r$, then the set $\{\x \in \R^K | \rank{A(\x)} < r\}$
has zero measure.
\end{lemma}
\begin{proof}
Remember that $\rank{A(\x)} \geq r $ iff there exits a non-zero $r \times r$ minor of $A(\x)$,
which is polynomial in the entries of $A(\x)$, and so it is polynomial in $\x$ as well. Let $c = {M \choose r} \cdot {N \choose r}$
be the number of minors in $A$, denote the minors by $\{f_i(\x)\}_{i=1}^c$, and define the polynomial
function $f(\x) = \sum_{i=1}^c f_i(\x)^2$. It thus holds that $f(\x) = 0$ iff for all $i \in [c]$ it holds that $f_i(\x) = 0$,
i.e. $f(\x) = 0$ iff $\rank{A(\x)} < r$.

Now, $f(\x)$ is a polynomial in the entries of $\x$, and so it either vanishes on a set of zero measure, or
it is the zero polynomial~(see \citet{caron2005zero} for proof). Since we assumed that there exists $\x \in \R^K$ s.t.
$\textrm{rank}(A(\x)) \geq r$, the latter option is not possible.
\end{proof}

Following the work of~\citet{expressive_power}, our main proof relies on following notations and facts:
\begin{itemize}
\item We denote by $[\A]$ the matricization of an $N$-order tensor $\A$ (for simplicity, $N$ is assumed
         to be even), where rows and columns correspond to odd and even modes, respectively. Specifically,
         if $\A \in \R^{M_1 \times \cdots M_N}$, the matrix $[\A]$ has $M_1 \cdot M_3 \cdot \ldots \cdot M_{N-1}$ rows
         and $M_2 \cdot M_4 \cdot \ldots \cdot M_N$ columns, rearranging the entries of the tensor such that
         $\A_{d_1,\ldots,d_N}$ is stored in row index
         $1 + \sum_{i=1}^{\nicefrac{N}{2}}(d_{2i-1} - 1) \prod_{j=i+1}^{\nicefrac{N}{2}} M_{2j-1}$ and column index
         $1 + \sum_{i=1}^{\nicefrac{N}{2}}(d_{2i} - 1) \prod_{j=i+1}^{\nicefrac{N}{2}} M_{2j}$.
         Additionally, the matricization is a linear operator, i.e. for all scalars $\alpha_1,\alpha_2$ and tensors
         $\A_1,\A_2$ with the order and dimensions in every mode, it holds that
         $[\alpha_1 \A_1 + \alpha_2 \A_2] = \alpha_1[\A_1] + \alpha_2 [\A_2]$.
\item The relation between the Kronecker product (denoted by $\odot$) and the tensor product (denoted
         by $\otimes$) is given by $[\A \otimes \B] = [\A] \odot [\B]$.
\item For any two matrices $A$ and $B$, it holds that $\rank{A \odot B} = \rank{A} \cdot \rank{B}$.
\item Let $Z$ be the CP-rank of $\A$, then it holds that $\rank{[\A]}~\leq~Z$ (see~\citep{expressive_power} for proof).
\end{itemize}

\begin{proof}[Proof of theorem~\ref{thm:tensor_rank}]
Stemming from the above stated facts, to show that the CP-rank of $\A^y$ is at least $r^{\nicefrac{N}{2}}$,
it is sufficient to examine its matricization $[\A^y]$ and prove that $\rank{[\A^y]}\geq~r^{\nicefrac{N}{2}}$.

Notice from the construction of $[\A^y]$, according to the recursive formula of the HT-decomposition, that its entires
are polynomial in the parameters of the decomposition, its dimensions are $M^{\nicefrac{N}{2}}$ each and that
$1 \leq r^{\nicefrac{N}{2}} \leq M^{\nicefrac{N}{2}}$.
In accordance with the discussion on the measure of simplex spaces, for each vector parameter
$\aaa^{l,j,\gamma} \in \triangle^{r_{l-1} - 1}$, we instead examine its projection
$\tilde{\aaa}^{l,j,\gamma} = p(\aaa^{l,j,\gamma}) \in \R^{r_{l-1}-1}$, and notice that $p^{-1}(\tilde{\aaa}^{l,j,\gamma})$
is a polynomial mapping\footnote{As we mentioned earlier, $p$ is invertible only over $p(\triangle^k)$, for which its inverse
is given by $p^{-1}(x_1,\ldots,x_k) = (x_1,\ldots,x_k,1-\sum_{i=1}^k x_i)$. However, to simplified the proof and notations,
we use $p^{-1}$ as defined here over the entire range $\R^{k-1}$, even where it does not serve as the inverse of~$p$.}
w.r.t. $\tilde{\aaa}^{l,j,\gamma}$. Thus, $[\A^y]$ is a polynomial mapping w.r.t. the projected parameters
$\{\tilde{\aaa}^{l,j,\gamma}\}_{l,j,\gamma}$, and using lemma~\ref{lemma:rank_everywhere} it is sufficient to show
that there exists a set of parameters for which $\rank{[\A^y]} \geq r^{\nicefrac{N}{2}}$.

Denoting for convenience $\phi^{L,1,1}:=\A^y$ and $r_L=1$, we will construct by induction over $l=1,...,L$
a set of parameters, $\{\aaa^{l,j,\gamma}\}_{l,j,\gamma}$, for which the ranks of the matrices
$\{[\phi^{l,j,\gamma}]\}_{j\in[\nicefrac{N}{2^l}],\gamma\in[r_l]}$ are at least $r^{\nicefrac{2^l}{2}}$,
while enforcing the simplex constraints on the parameters.
More so, we'll construct  these parameters s.t. $\aaa^{l,j,\gamma} = \aaa^{l,\gamma}$, thus
proving both the "unshared" and "shared" cases.

For the case $l=1$ we have:
$$\phi^{1,j,\gamma}= \sum_{\alpha=1}^{r_0} a_\alpha^{1,j,\gamma} \aaa^{0,2j-1,\alpha} \otimes  \aaa^{0,2j,\alpha}$$
and let $a^{1,j,\gamma}_\alpha = \frac{1_{\alpha \leq r}}{r}$ and $a^{0,j,\alpha}_i = 1_{\alpha = i}$ for all
$i,j,\gamma$ and $\alpha \leq M$, and $a^{0,j,\alpha}_i = 1_{i=1}$ for all $i$ and $\alpha > M$, and so
$$[\phi^{1,j,\gamma}]_{i,j} = \begin{cases}
\nicefrac{1}{r} & i = j \wedge i \leq r \\
0 & Otherwise
\end{cases}
$$
which means $\rank{[\phi^{1,j,\gamma}]} = r$, while preserving the simplex constraints,
which proves our inductive hypothesis for $l=1$.

Assume now that $\rank{[\phi^{l-1,j',\gamma'}]} \geq r^{\nicefrac{2^{l-1}}{2}}$ for all 
$j'\in[\nicefrac{N}{2^{l-1}}]$ and $\gamma'\in[r_{l-1}]$.  For some specific choice of $j\in[\nicefrac{N}{2^l}]$ 
and $\gamma\in[r_l]$ we have:
\begin{align*}
&&\phi^{l,j,\gamma} = \sum_{\alpha=1}^{r_{l-1}} a_\alpha^{l,j,\gamma} 
\phi^{l-1,2j-1,\alpha} \otimes  \phi^{l-1,2j,\alpha} \\
&&\implies [\phi^{l,j,\gamma}] = \sum_{\alpha=1}^{r_{l-1}} a_\alpha^{l,j,\gamma} 
[\phi^{l-1,2j-1,\alpha}] \odot [\phi^{l-1,2j,\alpha}]
\end{align*}
Denote $M_\alpha := [\phi^{l-1,2j-1,\alpha}] \odot [\phi^{l-1,2j,\alpha}]$ for $\alpha=1,...,r_{l-1}$.
By our inductive assumption, and by the general property $\rank{A \odot B}~=~\rank{A}~\cdot~\rank{B}$,
we have that the ranks of all matrices $M_\alpha$ are at least 
$r^{\nicefrac{2^{l-1}}{2}}\cdot r^{\nicefrac{2^{l-1}}{2}}=r^{\nicefrac{2^l}{2}}$.  Writing 
$[\phi^{l,j,\gamma}] = \sum_{\alpha=1}^{r_{l-1}} a_\alpha^{l,j,\gamma} \cdot M_\alpha$, and noticing that
$\{M_\alpha\}$ do not depend on $\aaa^{l,j,\gamma}$, we simply pick
$a^{l,j,\gamma}_\alpha = 1_{\alpha = 1}$, and thus $\phi^{l,j,\gamma} = M_1$, which is of rank
$r^{\nicefrac{2^l}{2}}$. This completes the proof of the theorem.
\end{proof}

From the perspective of ConvACs with simplex constraints, theorem~\ref{thm:tensor_rank} leads to the following corollary:

\begin{corollary}
Assume the mixing components $\mathcal{M}~=~\{f_i(\x)~\in~L^2(\R^2) \cap L^1(\R^s)\}_{i=1}^M$ are square
integrable\footnote{It is important to note that most commonly used distribution functions are square
integrable, e.g. most members of the exponential family such as the Gaussian distribution.}
probability density functions, which form a linearly independent set.
Consider a deep ConvAC model with simplex constraints of polynomial size whose parameters are drawn
at random by some continuous distribution. Then, with probability~1, the distribution realized by this network
requires an exponential size in order to be realized (or approximated w.r.t. the $L^2$ distance) by the shallow
ConvAC model with simplex constraints. The claim holds regardless of whether the parameters of the deep
model are shared or not.
\end{corollary}
\begin{proof}
Given a coefficient tensor $\A$, the CP-rank of $\A$ is a lower bound on the number of channels
(of its next to last layer) required to represent that tensor by the ConvAC following the
CP factorization. Additionally, since the mixing components are linearly
independent, their products $\{\prod_{i=1}^N f_i(\x_i) | f_i \in \mathcal{M}\}$ are linearly
independent as well, which entails that any distribution representable by the generative variant of ConvAC with
mixing components $\mathcal{M}$ has a unique coefficient tensor $\A$. From theorem~\ref{thm:tensor_rank},
the set of parameters of a deep ConvAC model (under the simplex constraints) with a coefficient tensor of a
polynomial CP-rank~--~the requirement for a polynomially-sized shallow ConvAC model with simplex constraints
realizing that same distribution exactly~--~forms a set of measure zero.

It is left to prove, that not only is it impossible to exactly represent a distribution with an exponential coefficient
tensor by a shallow model, it is also impossible to approximate it. This follows directly from lemma~7 in
appendix~B of~\citet{expressive_power}, as our case meets the requirement of that lemma.
\end{proof}

\section{Proof for the Optimality of Marginalized Bayes Predictor}\label{app:mbayes_proof}
In this section we give short proofs for the claims from sec.~\ref{sec:missing_data}, on the
optimality of the marginalized Bayes predictor under missing-at-random~(MAR) distribution, when
the missingness mechanism is unknown, as well as the general case when we do not add additional
assumptions. In addition, we will also present a counter example proving data imputation results lead to
suboptimal classification performance. We begin by introducing several notations that augment the
notations already introduced in the body of the article.

Given a specific mask realization $\m \in \{0,1\}^s$, we use the following notations to denote partial
assignments to the random vector $\X$. For the observed indices of $\X$, i.e. the indices for which $m_i = 1$,
we denote a partial assignment by $\X \setminus \m = \x_o$, where $\x_o \in \R^{d_o}$ is a vector of
length $d_o$ equal to the number of observed indices. Similarly, we denote by $\X \cap \m = \x_m$ a
partial assignment to the missing indices according to $\m$, where $\x_m \in \R^{d_m}$ is a vector of
length $d_m$ equal to the number of missing indices. As an example of the notation, for given realizations
$\x \in \R^s$ and $\m \in \{0,1\}^s$, we defined in sec.~\ref{sec:missing_data} the event $o(\x,\m)$,
which using current notation is marked by the partial assignment $\X \setminus \m = \x_o$ where $\x_o$
matches the observed values of the vector $\x$ according to $\m$.

With the above notations in place, we move on to prove claim~\ref{claim:optimal_rule}, which
describes the general solution to the optimal prediction rule given both the data and missingness
distributions, and without adding any additional assumptions.

\begin{proof}[Proof of claim~\ref{claim:optimal_rule}]
Fix an arbitrary prediction rule $h$. We will show that $L(h^*) \leq L(h)$, where $L$ is the expected
0-1 loss. 
\begin{align*}
&1 - L(h) {=} E_{(\x,\m,y)\sim(\X,\M,\Y)}[1_{h(\x \odot \m) = y}] \\
&{=} {\sum_{\m \in \{0,1\}^s}} {\sum_{y \in [k]}} {\int_{\R^s}} \PP(\M{=}\m, \X{=}\x, \Y{=}y) 1_{h(\x {\odot} \m) {=} y} d\x \\
&{=} {\sum_{\m \in \{0,1\}^s}} {\sum_{y \in [k]}} {\int_{\R^{d_o}}} {\int_{\R^{d_m}}} \\
&\phantom{{=}}\PP(\M{=}\m, \X{\setminus}\m {=} \x_o, \X{\cap}\m {=} \x_m, \Y{=}y) 1_{h(\x {\otimes} \m) {=} y} d\x_o d\x_m \\
&{=_1} {\sum_{\m \in \{0,1\}^s}} {\sum_{y \in [k]}} {\int_{\R^{d_o}}} 1_{h(\x {\odot} \m) {=} y} d\x_o \\
&\phantom{{=_1}}{\int_{\R^{d_m}}} \PP(\M{=}\m, \X{\setminus}\m {=} \x_o, \X{\cap}\m {=} \x_m, \Y{=}y) d\x_m \\
&{=_2} {\sum_{\m \in \{0,1\}^s}} {\sum_{y \in [k]}} {\int_{\R^{d_o}}} 1_{h(\x {\odot} \m) {=} y} \PP(\M{=}\m, \X{\setminus}\m{=}\x_o, \Y{=}y) d\x_o \\
&{=_3} {\sum_{\m \in \{0,1\}^s}} {\int_{\R^{d_o}}} \PP(\X{\setminus}\m{=}\x_o) {\sum_{y \in [k]}} 1_{h(\x {\odot} \m) {=} y} \PP(\Y{=}y | \X{\setminus}\m{=}\x_o) \\
&\phantom{{=_3} }\PP(\M{=}\m | \X{\setminus}\m{=}\x_o, \Y{=}y) d\x_o \\
&{\leq_4} {\sum_{\m \in \{0,1\}^s}} {\int_{\R^{d_o}}} \PP(\X{\setminus}\m{=}\x_o) {\sum_{y \in [k]}} 1_{h^*(\x {\odot} \m) {=} y} \PP(\Y{=}y | \X{\setminus}\m{=}\x_o) \\
&\phantom{\leq_4}\PP(\M{=}\m | \X{\setminus}\m{=}\x_o, \Y{=}y) d\x_o \\
&{=} 1- L(h^*)
\end{align*}
Where (1) is because the output of $h(\x \odot \m)$ is independent of the missing values, (2) by marginalization,
(3) by conditional probability definition and (4) because by definition $h^*(\x \odot \m)$ maximizes the expression
$\PP(\Y{=}y | \X{\setminus}\m{=}\x_o) \PP(\M{=}\m | \X{\setminus}\m{=}\x_o, \Y{=}y)$ w.r.t. the possible values of $y$ for fixed vectors~$\m$ and~$\x_o$.
Finally, by replacing integrals with sums, the proof holds exactly the same when instances ($\X$) are discrete.
\end{proof}

We now continue and prove corollary~\ref{corollary:mar}, a direct implication of claim~\ref{claim:optimal_rule} which shows that in the MAR setting, the missingness distribution can be ignored, and the optimal prediction rule is given by the marginalized Bayes predictor.

\begin{proof}[Proof of corollary~\ref{corollary:mar}]
Using the same notation as in the previous proof, and denoting by $\x_o$ the partial vector containing the observed values
of $\x \odot \m$, the following holds:
\begin{align*}
&\PP(\M{=}\m|o(\x,\m), \Y{=}y) := \PP(\M{=}\m|\X {\setminus} \m {=} \x_o, \Y{=}y) \\
&{=} \int_{\R^{d_m}} \PP(\M{=}\m, \X \cap \m {=} \x_m | \X {\setminus} \m {=} \x_o, \Y{=}y) d\x_m \\
&{=} \int_{\R^{d_m}} \PP(\X {\cap} \m {=} \x_m | \X {\setminus} \m {=} \x_o, \Y{=}y)\\
&\phantom{{=}} \cdot \PP(\M{=}\m | \X {\cap} \m {=} \x_m, \X {\setminus} \m {=} \x_o, \Y{=}y) d\x_m \\
&{=_1} \int_{\R^{d_m}} \PP(\X {\cap} \m {=} \x_m | \X {\setminus} \m {=} \x_o, \Y{=}y) \\
&\phantom{=_1}        \cdot \PP(\M{=}\m | \X {\cap} \m {=} \x_m, \X {\setminus} \m {=} \x_o) d\x_m \\
&{=_2} \int_{\R^{d_m}} \PP(\X {\cap} \m {=} \x_m | \X {\setminus} \m {=} \x_o, \Y{=}y)
       \cdot \PP(\M{=}\m | \X {\setminus} \m {=} \x_o) d\x_m \\
&{=} \PP(\M{=}\m | \X {\setminus} \m {=} \x_o) {\int_{\R^{d_m}}} \PP(\X {\cap} \m {=} \x_m | \X {\setminus} \m {=} \x_o, \Y{=}y) d\x_m \\
&{=} \PP(\M{=}\m | o(\x,\m))
\end{align*}
Where $(1)$ is due to the independence assumption of the events $\Y=y$ and $\M=\m$ conditioned on $\X=\x$, while
noting that $(\X \setminus \m = x_o) \wedge (\X \cap \m = x_m)$ is a complete assignment of $\X$.
$(2)$ is due to the MAR assumption, i.e. that for a given $\m$ and $\x_o$ it holds for all $\x_m \in \R^{d_m}$: 
$$\PP(\M{=}\m|\X {\setminus} \m {=} \x_o, \X {\cap} \m {=} \x_m) = \PP(\M{=}\m|\X {\setminus} \m {=} \x_o)$$
We have shown that $\PP(\M{=}\m|o(\x,\m), \Y=y)$ does not depend on $y$, and thus does not affect the optimal prediction rule in claim~\ref{claim:optimal_rule}.  It may therefore be dropped, and we obtain the marginalized Bayes predictor.
\end{proof}

Having proved that in the MAR setting, classification through marginalization leads to optimal performance, we now move on to
show that the same is not true for classification through data-imputation. Though there are many methods
to perform data-imputation, i.e. to complete missing values given the observed ones, all of these methods
can be seen as the solution of the following optimization problem, or more typically its approximation:
\begin{align*}
g(\x \odot \m) = \argmax_{\x' \in \R^s \wedge \forall i: m_i = 1 \rightarrow x'_i = x_i} \PP(\X=\x') 
\end{align*}
Where $g(\x \odot \m)$ is the most likely completion of $\x \odot \m$. When data-imputation is carried out
for classification purposes, one is often interested in data-imputation conditioned on a given class $Y=y$, i.e.:
\begin{align*}
g(\x \odot \m; y) = \argmax_{\x' \in \R^s \wedge \forall i: m_i = 1 \rightarrow x'_i = x_i} \PP(\X=\x'|\Y=y) 
\end{align*}
Given a classifier $h:\R^s \to [K]$ and an instance $\x$ with missing values according to $\m$,
classification through data-imputation is simply the result of applying $h$ on the output of $g$. When $h$
is the optimal classifier for complete data, i.e.~the Bayes predictor, we end up with one of the following prediction rules:
\begin{align*}
\textrm{Unconditional:} &\, h(\x \odot \m) = \argmax_y \PP(\Y = y | \X = g(\x \odot \m))\\
\textrm{Conditional:} & \, h(\x \odot \m) = \argmax_y \PP(\Y = y | \X = g(\x \odot \m; y))
\end{align*}

\begin{table}
\centering
\begin{tabular}{ccccc}
\toprule
$X_1$ & $X_2$ & $Y$ &     Weight       & Probability ($\epsilon = 10^{-4}$) \\ 
\midrule
  $0$   &  $0$   & $0$ & $1 - \epsilon$ & $16.665\%$ \\
  $0$   &  $1$   & $0$ &        $1$          & $16.667\%$ \\
  $1$   &  $0$   & $0$ & $1 - \epsilon$ & $16.665\%$ \\
  $1$   &  $1$   & $0$ &         $1$         & $16.667\%$ \\
  $0$   &  $0$   & $1$ &         $0$         & $  0.000\%$ \\
  $0$   &  $1$   & $1$ & $1 + \epsilon$ & $16.668\%$ \\
  $1$   &  $0$   & $1$ &         $0$         & $  0.000\%$ \\
  $1$   &  $1$   & $1$ & $1 + \epsilon$ & $16.668\%$ \\  
\bottomrule         
\end{tabular}
\caption{Data distribution over the space $\X~{\times}~\Y~=~\{0,1\}^2~{\times}~\{0,1\}$ that serves as
             the example for the sub-optimality of classification through data-imputation (proof of claim~\ref{claim:data_imp_subopt}).}
\label{table:counter_example}
\end{table}

\begin{claim} \label{claim:data_imp_subopt}
There exists a data distribution $\D$ and MAR missingness distribution $\Q$ s.t. the accuracy of classification
through data-imputation is almost half the accuracy of the optimal marginalized Bayes predictor, with an absolute
gap of more than $33$ percentage points.
\end{claim}
\begin{proof}
For simplicity, we will give an example for a discrete distribution over the binary set
$\X~{\times}~\Y~{=}~\{0,1\}^2~{\times}~\{0,1\}$. Let $1~{>}~\epsilon~{>}~0$ be some small positive number, and we define $\D$
according to table~\ref{table:counter_example}, where each triplet $(x_1,x_2,y) \in \X{\times}\Y$ is assigned a positive weight,
which through normalization defines a distribution over $\X{\times}\Y$. The missingness distribution $\Q$ is defined
s.t. $P_\Q(M_1 = 1, M_2 = 0 | X = \x) = 1$ for all $\x \in \X$, i.e. $X_1$ is always observed and $X_2$ is always missing, which
is a trivial MAR distribution. Given the above data distribution $\D$, we can easily calculate the exact accuracy of the optimal
data-imputation classifier and the marginalized Bayes predictor under the missingness distribution $\Q$, as well as the standard
Bayes predictor under full-observability. First notice that whether we apply conditional or unconditional data-imputation, and
whether $X_1$ is equal to $0$ or $1$, the completion will always be $X_2 = 1$ and the predicted class will always be $Y=1$.
Since the data-imputation classifiers always predict the same class $Y=1$ regardless of their input, the probability of success
is simply the probability $P(Y=1) = \frac{1 + \epsilon}{3}$ (for $\epsilon = 10^{-4}$ it equals approximately $33.337\%$).
Similarly, the marginalized Bayes predictor always predicts $Y=0$ regardless of its input, and so its probability of success is
$P(Y = 0) = \frac{2 - \epsilon}{3}$ (for $\epsilon = 10^{-4}$ it equals approximately $66.663\%$), which is almost double
the accuracy achieved by the data-imputation classifier. Additionally, notice that the marginalized Bayes predictor achieves
almost the same accuracy as the Bayes predictor under full-observability, which equals exactly $\frac{2}{3}$.

\end{proof}

% EFFICIENT MARGINALIZATION WITH TMMS
\section{Efficient Marginalization with Tensorial Mixture Models} \label{sec:missing_data:margin}

As discussed above, with generative models optimal classification under missing data (in the MAR setting)
is oblivious to the specific missingness distribution. However, it requires tractable marginalization over missing values. 
In this section we show that TMMs bring forth extremely efficient marginalization, requiring only a
single forward pass through the corresponding ConvAC.

Recall from sec.~\ref{sec:model} and~\ref{sec:training} that a TMM classifier realizes the following form:
\begin{align}
P(\x_1,\ldots,\x_N|Y{=}y) &= \sum\nolimits_{d_1,\ldots,d_N}^M P(d_1,\ldots,d_N|Y{=}y) \prod\nolimits_{i=1}^{N} P(\x_i | d_i; \theta_{d_i})
\label{eq:mc_tmm}
\end{align}
Suppose now that only the local structures $\x_{i_1}\ldots\x_{i_V}$ are observed, and we would like to
marginalize over the rest. Integrating eq.~\ref{eq:mc_tmm} gives:
\begin{align*}
P(\x_{i_1},\ldots,\x_{i_V}|Y{=}y) &= \sum\nolimits_{d_1,\ldots,d_N}^M P(d_1,\ldots,d_N|Y{=}y) \prod\nolimits_{v=1}^{V} P(\x_{i_v} | d_{i_v}; \theta_{d_{i_v}})
\end{align*}
from which it is evident that the same network used to compute $P(\x_1,\ldots,\x_N|Y{=}y)$, can be used
to compute $P(\x_{i_1},\ldots,\x_{i_V}|Y{=}y)$~--~all it requires is a slight adaptation of the representation
layer. Namely, the latter would represent observed values through the usual likelihoods, whereas missing
(marginalized) values would now be represented via constant ones:
\begin{align*}
\textrm{rep}(i,d) &= \begin{cases}
              \qquad1 & \textrm{, $\x_i$ is missing (marginalized)} \\
              P(\x_i|d;\Theta) & \textrm{, $\x_i$ is visible (not marginalized)}
       \end{cases}
\end{align*}
More generally, to marginalize over individual coordinates of the local structure $\x_i$, 
it is sufficient to replace $\textrm{rep}(i,d)$ by its respective marginalized mixing component.

To conclude, with TMMs marginalizing over missing values is just as efficient as plain inference~--~requires
only a single pass through the corresponding network. Accordingly, the marginalized Bayes predictor (eq.~\ref{eq:mbayes})
is realized efficiently, and classification under missing data (in the MAR setting) is optimal (under generative assumption), regardless of the missingness distribution.

\section{Extended Discussion on Generative Models Based on Neural Networks} \label{app:extended_related_works}

There are many generative models realized through neural networks, and convolutional networks in particular.
Of these models, one of the most successful to date is the method of Generative Adversarial
Networks~\citep{Goodfellow:2014td}, where a network is trained to generate instances from the data
distribution, through a two-player mini-max game. While there are numerous applications for learning to generate
data points, e.g. inpainting and super-resolution, it cannot be used for computing the likelihood of the data.
Other generative networks do offer inference, but only approximate.
Variational Auto-Encoders~\citep{Kingma:2013tz}
use a variational lower-bound on the likelihood function. GSNs~\citep{Bengio:2013vx},
DPMs~\citep{SohlDickstein:2015vq} and MPDBMs~\citep{Goodfellow:2013vm} are additional methods along this line. The latter is especially
noteworthy for being a generative classifier that can approximate the marginal likelihoods conditioned on each
class, and for being tested on classification under missing data.

Some generative neural networks are capable of tractable inference, but not of tractable marginalization.
\citet{Dinh:2014vu} suggest a method for designing neural networks that realize an invertible transformation
from a simple distribution to the data distribution. Inverting the network brings forth tractable inference, yet partial
integration of its density function is still intractable. Another popular method for tractable inference, central to
both PixelRNN~\citep{vandenOord:2016um} and NADE~\citep{JMLR:v17:16-272}, is the factorization of the probability distribution according
to $\PP(x_1,\ldots,x_d) = \prod_{i=1}^d \PP(x_i | x_{i-1},\ldots,x_1)$, and realization of $\PP(x_i|x_{i-1},\ldots,x_1)$ as
a neural network. Based on this construction, certain
marginal distributions are indeed tractable to compute, but most are not. Orderless-NADE partially addresses this issue
by using ensembles of models over different orderings of its input. However, it can only estimate the marginal distributions,
and has no classifier analogue that can compute class-conditional marginal likelihoods, as required
for classification under missing data.

\section{Image Generation and Network Visualization}\label{app:exp:vis}

\begin{figure}
\centering
\includegraphics[width=0.85\columnwidth]{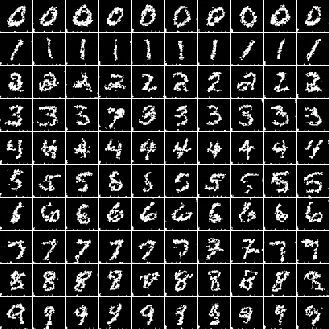}
\caption{Generated digits samples from the HT-TMM model trained on the MNIST dataset.}
\label{fig:fulldigits}
\end{figure}

\begin{figure*}
\centering
\includegraphics[width=0.3\linewidth]{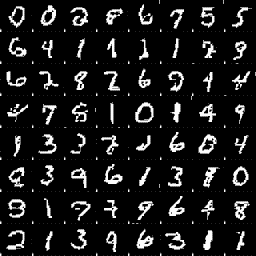}
\includegraphics[width=0.3\linewidth]{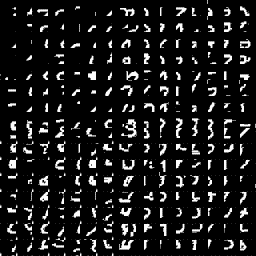}
\includegraphics[width=0.3\linewidth]{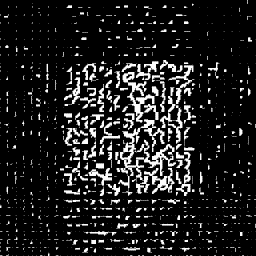}
\caption{Visualization of the HT-TMM model. Each of the  images above visualize
             a different layer of the model and consists of several samples generated from latent
             variables at different spatial locations conditioned on randomly selected channels. The
             leftmost image shows samples taken from the 5th layer which consists of just a single
             latent variable with 512 channels. The center image shows samples taken from the
             4th layer, which consists of $2 \times 2$ grid of latent variables with 256 channels each.
             The image is divided to 4 quadrants, each contains samples taken from the respective
             latent variable at that position. The rightmost image shows samples from the 3rd layer,
             which consists of $4 \times 4$ grid of latent variables with 128 channels, and the image
             is similarly spatial divided into different areas matching the latent variables of the layer.  }
\label{fig:visnet}
\end{figure*}

Following the graphical model perspective of our models allows us to not only
generate random instances from the distribution, but to also generate the most
likely patches for each neuron in the network, effectively explaining its role
in the classification process. We remind the reader that every neuron in the 
network corresponds to a possible assignment of a latent variable in the graphical
model. By looking for the most likely assignments for each of its child nodes in
the graphical tree model, we can generate a patch that describes that neuron.
Unlike similar suggested methods to visualize neural networks~\citep{Zeiler:2014fra},
often relying on brute-force search or on solving some optimization problem to
find the most likely image, our method emerges naturally from the probabilistic
interpretation of our model.

In fig.~\ref{fig:fulldigits}, we can see conditional samples generates for each
digit, while in fig.~\ref{fig:visnet} we can see a visualization of the top-level
layers of network, where each small patch matches a different neuron in the
network. The common wisdom of how ConvNets work is by assuming that
simple low-level features are composed together to create more and more complex
features, where each subsequent layer denotes features of higher
abstraction~--~the visualization of our network clearly demonstrate this hypothesis
to be true for our case, showing small strokes iteratively being composed into
complete digits.

\section{Detailed Description of the Experiments}\label{app:exp_details}

Experiments are meaningful only if they could be reproduced by other proficient individuals. Providing sufficient details
to enable others to replicate our results is the goal of this section. We hope to accomplish this by making our code
public, as well as documenting our experiments to a sufficient degree allowing for their reproduction from scratch.
Our complete implementation of the models presented in this paper, as well as our modifications to other open-source projects
and scripts used in the process of conducting our experiments, are available at our Github repository:
\githuburl{Generative-ConvACs}. We additionally wish to invite readers to contact the authors, if they deem
the following details insufficient in their process to reproduce our results. 

\subsection{Description of Methods}
In the following we give concise descriptions of each classification method we have used in our experiments.
The results of the experiment on MP-DBM~\citep{Goodfellow:2013vm} were taken directly from the paper and
were not conducted by us, hence we do not cover it in this section. We direct the reader to that article for exact
details on how to reproduce their results.

\subsubsection{Robust Linear Classifier}

In~\cite{Dekel:2008jb}, binary linear classifiers were trained by formulating their optimization as a quadric program
under the constraint that some of its features could be deleted, i.e. their original value was changed to zero.
While the original source code was never published, the authors have kindly agreed to share with us their
code, which we used to reproduced their results, but on larger datasets. The algorithm has only a couple
hyper-parameters, which were chosen by a grid-search through a cross-validation process. For details
on the exact protocol for testing binary classifiers on missing data, please see sec.~\ref{app:sec:exp:binary}.

\subsubsection{K-Nearest Neighbors}

K-Nearest Neighbors~(KNN) is a classical machine learning algorithm used for both regression and classification tasks.
Its underlying mechanism is finding the $k$~nearest examples (called neighbors) from the training set, $(\x_1,y_1),\ldots,(\x_k,y_k) \in S$,
according to some metric function~$d(\cdot,\cdot):\X \times \X \to \R_+$, after which a summarizing function~$f$ is applied to the targets
of the $k$~nearest neighbors to produce the output $y^* = f(y_1,\ldots,y_k)$. When KNN is used for classification, $f$~is typically
the majority voting function, returning the class found in most of the $k$ nearest neighbors.

In our experiments we use KNN for classification under missing data, where the training set consists of complete examples with no
missing data, but at classification time the inputs have missing values. Given an input with missing values $\x \odot \m$ and an
example $\x'$ from the training set, we use a modified Euclidean distance metric, where we compare the distance only against
the non-missing coordinates of $\x$, i.e. the metric is defined by~$d(\x', \x~\odot~\m)~=~\sum_{i:m_i = 1} \left(x'_i - x_i \right)^2$.
Through a process of cross-validation we have chosen $k=5$ for all of our experiments. Our implementation of KNN is based
on the popular \emph{scikit-learn} python library~\citep{scikit-learn}.

\subsubsection{Convolutional Neural Networks}
The most widespread and successful discriminative method nowadays are Convolutional Neural Networks~(ConvNets).
Standard ConvNets are represented by a computational graph consisted of different kinds of nodes, called layers,
with a convolutional-like operators applied to their inputs, followed by a non-linear point-wise activation function,
e.g. $\max(0,x)$ known as ReLU.

For our experiments on MNIST, both with and without missing data, we have used the LeNeT ConvNet
architecture~\citep{LeCun:1998hy} that is bundled with Caffe~\citep{Jia:2014up}, trained for 20,000 iterations using
SGD with $0.9$ momentum and $0.01$ base learning rate, which remained constant for 10,000 iterations, followed
by a linear decrease to $0.001$ for another 5,000 iterations, followed by a linear decrease to $0$ learning rate for
the remaining 5,000 iterations. The model also used $l_2$-regularization (also known as weight decay), which was
chosen through cross-validation for each experiment separately. No other modifications were made to the model or
its training procedure.

For our experiments on NORB, we have used an ensemble of 3 ConvNets, each using the following architecture:
$5{\times}5$ convolution with 128 output channels, $3{\times}3$ max pooling with stride 2,
ReLU activation, $5{\times}5$ convolution with 128 output channels, ReLU activation,
dropout layer with probability 0.5, $3{\times}3$ average pooling with stride 2, $5{\times}5$ convolution with
256 output channels, ReLU activation, dropout layer with probability 0.5, $3{\times}3$ average pooling with stride 2,
fully-connected layer with 768 output channels, ReLU activation, dropout layer with probability 0.5, and ends with
fully-connected layer with 5 output channels. The stereo images were represented as a two-channel input
image when fed to the network. During training we have used data augmentation consisting of randomly
scaling and rotation transforms. The networks were trained for 40,000 iterations using SGD with
$0.99$ momentum and $0.001$ base learning rate, which remained constant for 30,000 iterations, followed by
a linear decrease to $0.0001$ for 6000 iterations, followed by  a linear decrease to $0$ learning rate for the remaining
4,000 iterations. The model also used $0.0001$ weight decay for additional regularization. 

When ConvNets were trained on images containing missing values, we passed the network the original image with
missing values zeroed out, and an additional binary image as a separate channel, containing $1$ for missing
values at the same spatial position, and $0$ otherwise -- this missing data format is sometimes known as
\emph{flag data imputation}. Other formats for representing missing values were tested (e.g. just using zeros for
missing values), however, the above scheme performed significantly better than other formats. In our experiments,
we assumed that the training set was complete and missing values were only present in the test set. In order to
design ConvNets that are robust against specific missingness distributions, we have simulated missing values
during training, sampling a different mask of missing values for each image in each mini-batch. As covered in
sec.~\ref{sec:exp}, the results of training ConvNets directly on simulated missingness distributions resulted in
classifiers which were biased towards the specific distribution used in training, and performed worse on other
distributions compared to ConvNets trained on the same distribution.

In addition to training ConvNets directly on missing data, we have also used them as the classifier for testing
different data imputation methods, as describe in the next section.

\subsubsection{Classification Through Data Imputation}

The most common method for handling missing data, while leveraging available discriminative classifiers,
is through the application of \emph{data imputation}~--~an algorithm for the completion of missing
values~--~and then passing the results to a classifier trained on uncorrupted dataset. We have tested
five different types of data imputation algorithms:
\begin{itemize}
\item Zero data imputation: replacing every missing value by zero.
\item Mean data imputation: replacing every missing value by the mean value computed over the dataset.
\item Generative data imputation: training a generative model and using it to complete the missing values
         by finding the most likely instance that coincides with the observed values, i.e. solving the following
         \begin{align*}
         g(\x \odot \m) &= \argmax_{\x' \in \R^s \wedge \forall i, m_i = 1 \rightarrow x'_i = x_i} P(X=\x') 
         \end{align*}
         We have tested the following generative models:
         \begin{itemize}
         \item Generative Stochastic Networks~(GSN)~\citep{Bengio:2013vx}: We have used their original source code from
                  \url{https://github.com/yaoli/GSN}, and trained their example model on MNIST for 1000 epochs. Whereas in the
                  original article they have tested completing only the left or right side of a given image, we have modified their
                  code to support general masks. Our modified implementation can be found at \githuburl{GSN}.
         \item Non-linear Independent Components Estimation~(NICE)~\citep{Dinh:2014vu}: We have used their original source
                  code from \url{https://github.com/laurent-dinh/nice}, and trained it on MNIST using their example code without changes.
                  Similarly to our modification to the GSN code, here too we have adapted their code to support general masks
                  over the input. Additionally, their original inpainting code required 110,000 iterations, which we have reduced to
                  just 8,000 iterations, since the effect on classification accuracy was marginal. For the NORB dataset, we have
                  used their CIFAR10 example, with lower learning rate of $10^{-4}$. Our modified code can be 
                  found at \githuburl{nice}.
         \item Diffusion Probabilistic Models~(DPM)~\citep{SohlDickstein:2015vq}: We have user their original source code
                  from \url{https://github.com/Sohl-Dickstein/Diffusion-Probabilistic-Models}, and trained it on MNIST using their
                  example code without changes. Similarly to our modifications to GSN, we have add support for a general
                  mask of missing values, but other than that kept the rest of the parameters for inpainting unchanged.
                  For NORB we have used the same model as MNIST. We have tried using their CIFAR10 example for NORB,
                  however, it produced exceptions during training.
                  Our modified code can be found at \githuburl{Diffusion-Probabilistic-Models}.
         \end{itemize}
\end{itemize}

\subsubsection{Tensorial Mixture Models}

For a complete theoretical description of our model please see the body of the article. Our models were
implemented by performing all intermediate computations in log-space, using numerically aware operations.
In practiced, that meant our models were realized by the SimNets architecture~\citep{simnets1, simnets2},
which consists of Similarity layers representing gaussian distributions, MEX layers representing weighted
sums performed on log-space input and outputs, as well as standard pooling operations. The learned
parameters of the MEX layers are called \emph{offsets}, which represents the weights of the weighted sum,
but saved in log-space. The parameters of the MEX layers can be optionally shared between spatial regions,
or alternatively left with no parameter sharing at all. Additionally, when used to implement our generative
models, the offsets are normalized to have a soft-max~(i.e., $\log\left(\sum_i \exp(x_i)\right)$) of zero.

The network architectures we have tested in this article, consists of $M$ different Gaussian mixture components
with diagonal covariance matrices, over non-overlapping patches of the input of size $2 \times 2$, which were
implemented by a similarity layer as specified by the SimNets architecture, but with an added gaussian normalization term.

We first describe the architectures used for the MNIST dataset.
For the CP-TMM model, we used $M=800$, and following the similarity layer is a $1 \times 1$~MEX layer with no parameter sharing
over spatial regions and $10$ output channels. The model ends with a global sum pooling operation, followed by another $1 \times 1$ MEX
layer with $10$ outputs, one for each class. The HT-TMM model starts with the similarity layer with $M=32$, followed by a sequence
of four pairs of $1 \times 1$~MEX layer followed by $2 \times 2$ sum pooling layer, and after the pairs and additional
$1 \times 1$~MEX layer lowering the  outputs of the model to $10$ outputs as the number of classes. The number of output
channels for each MEX layer are as follows 64-128-256-512-10. All the MEX layers in this network do not use parameter sharing,
except the first MEX layer, which uses a repeated sharing pattern of $2 \times 2$ offsets, that analogous to a
$2 \times 2$~convolution layer with stride $2$. Both models were trained with the losses described in sec.~\ref{sec:training},
using the Adam SGD variant for optimizing the parameters, with a base learning rate of $0.03$, and $\beta_1 = \beta_2 = 0.9$.
The models were trained for 25,000 iterations, where the learning rate was dropped by $0.1$ after 20,000 iterations.

For the NORB dataset, we have trained only the HT-TMM model with  $M=128$ for the similarity layer. The MEX layers use
the same parameter sharing scheme as the one for MNIST, and the number of output channels for each MEX layer are
as follows: 256-256-256-512-5. Training was identical to the MNIST models, with the exception of using 40,000 iterations
instead of just 25,000. Additionally, we have used an ensemble of 4 models trained separately, each trained using a different
generative loss weight (see below for more information). We have also used the same data augmentation methods (scaling and rotation)
which were used in training the ConvNets for NORB used in this article.

The standard $L_2$ weight regularization~(sometimes known as weight decay) did not work well on our models,
which lead us to adapt it to better fit to log-space weights, by minimizing $\lambda \sum_i \left(\exp\left(x_i\right)\right)^2$ instead
of $\lambda || \x ||_2~=~\lambda\sum_i \x_i^2$, where the parameter $\lambda$ was chosen through cross-validation.
Additionally, since even with large values of $\lambda$ our model was still overfitting, we have added another form
of regularization in the form of \emph{random marginalization} layers. A random marginalization layer, is similar in concept
to dropout, but instead of zeroing activations completely in random, it choses spatial locations at random, and then zero out
the activations at those locations for all the channels. Under our model, zeroing all the activations in a layer at a specific location, is
equivalent to marginalizing over all the inputs for the receptive field for that respective location. We have used random marginalization
layers in between all our layers during training, where the probability for zeroing out activations was chosen through cross-validation
for each layer separately. Though it might raise concern that random marginalization layers could lead to biased results toward
the missingness distributions we have tested it on, in practice the addition of those layers only helped improve our results under
cases where only few pixels where missing.

Finally, we wish to discuss a few optimization tricks which had a minor effects compared to the above, but were nevertheless
very useful in achieving slightly better results. First, instead of optimizing directly the objective defined by eq.~\ref{eq:objective},
we add smoothing parameter $\beta$ between the two terms, as follows:
\begin{align*}
\Theta^* &= \argmin_\Theta -\sum_{i=1}^{\abs{S}} \log \frac{e^{N_\Theta(X^{(i)};Y^{(i)})}}{\sum\nolimits_{y=1}^{K}e^{N_\Theta(X^{(i)};y)}} \\
&\phantom{=\argmin_\Theta}- \beta\sum_{i=1}^{\abs{S}} \log \sum_{y=1}^{K}e^{N_\Theta(X^{(i)};y)}
\end{align*}
setting $\beta$ too low diminish the generative capabilities of our models, while setting it too high diminish the discriminative 
performance. Through cross-validation, we decided on the value $\beta=0.01$ for the models trained on MNIST, while for NORB
we have used a different value of $\beta$ for each of the models, ranging in $\{0.01,0.1,0.5,1\}$. Second, we found that
performance increased if we normalized activations before applying the $1 \times 1$ MEX operations. Specifically, we
calculate the soft-max over the channels for each spatial location which we call the activation norm, and then subtract it
from every respective activation. After applying the MEX operation, we add back the activation norm. Though might not
be obvious at first, subtracting a constant from the input of a MEX operation and adding it to its output is equivalent does
not change the mathematical operation. However, it does resolve the numerical issue of adding very large activations to
very small offsets, which might result in a loss of precision. Finally, we are applying our model in different translations of
the input and then average the class predictions. Since our model can marginalize over inputs, we do not need to crop
the original image, and instead mask the unknown parts after translation as missing. Applying a similar trick to standard
ConvNets on MNIST does not seem to improve their results. We believe this method is especially fit to our model, is
because it does not have a natural treatment of overlapping patches like ConvNets do, and because it is able to
marginalize over missing pixels easily, not limiting it just to crop translation as is typically done.

\subsection{Description of Experiments}

In this section we will give a detailed description of the protocol we have used during our experiments.

\subsubsection{Binary Digit Classification under Feature Deletion Missing Data}\label{app:sec:exp:binary}

This experiment focuses on the binary classification problem derived from MNIST, by limiting the
number of classes to two different digits at a time. We use the same non-zero feature deletion
distribution as suggested by \citet{Globerson:2006jv}, i.e. for a given image we uniformly
sample a set of $N$ non-zero pixels from the image (if the image has less than $N$ non-zero
pixels then they are non-zero pixels are chosen), and replace their values with zeros. This
type of missingness distribution falls under the MNAR type defined in sec.\ref{sec:missing_data}.

We test values of $N$ in $\{0, 25, 50, 75, 100, 125, 150\}$. For a given value of $N$, we train
a separate classifier on each digit pair classifier on a randomly picked subset of the dataset
containing 300 images per digit (600 total). During training we use a fixed validation set with
1000 images per digit. After picking the best classifier according to the validation set, the
classifier is tested against a test set with a 1000 images per digits with a randomly chosen
missing values according to the value of $N$. This experiment is repeated 10 times for each
digit pair, each time using a different subset for the training set, and a new corrupted test set.
After conducting all the different experiments, all the accuracies are averaged for each value
of $N$, which are reported in table~\ref{table:exp_shamir}.

\subsubsection{Multi-class Digit Classification under MAR Missing Data}

This experiment focuses on the complete multi-class digit classification of the MNIST dataset,
in the presence of missing data according to different missingness distributions. Under this setting,
only the test set contains missing values, whereas the training set does not. We test two kinds of
missingness distributions, which both fall under the MAR type defined in sec.\ref{sec:missing_data}.
The first kind, which we call \emph{i.i.d. corruption}, each pixel is missing with a fixed probability $p$.
the second kind, which we call \emph{missing rectangles corruption}, The positions of $N$ rectangles
of width $W$ or chosen uniformly in the picture, where the rectangles can overlap one another.
During the training stage, the models to be tested are not to be biased toward the specific missingness
distributions we have chosen, and during the test stage, the same classifier is tested against all
types of missingness distributions, and without supplying it with the parameters or type of the missingness
distribution it is tested against. This rule prevent the use of ConvNets trained on simulated missingness
distributions. To demonstrate that the latter lead to biased classifiers, we have conducted a separate
experiment just for ConvNets, where the previous rule is ignored, and we train a separate ConvNet
classifier on each type and parameter of the missingness distributions we have used. We then
tested each of those ConvNets on all other missingness distributions, the results of which are
in fig.~\ref{fig:convnet}, which confirmed our hypothesis.

\end{document}